\documentclass[twoside]{article}

\usepackage[accepted]{aistats2023}

\usepackage[utf8]{inputenc} 
\usepackage[T1]{fontenc}    
\usepackage{hyperref}       
\usepackage{url}            
\usepackage{booktabs}       
\usepackage{multirow}
\usepackage{amsfonts}       
\usepackage{nicefrac}       
\usepackage{microtype}      
\usepackage{xcolor}         
\usepackage{graphicx}
\usepackage{subcaption}
\usepackage{appendix}
\usepackage{float}

\usepackage[
  sortcites,
  doi=false,
  url=false,
  isbn=false,
  eprint=true,
  date=year,
  maxbibnames=10,
  mincitenames=1,
  maxcitenames=2,
  uniquename=init,
  giveninits,
  uniquelist=minyear,
  style=authoryear,
  dashed=false
]{biblatex}
\AtEveryBibitem{%
  \clearname{editor}
}
\usepackage{amsmath}
\usepackage{amssymb}
\usepackage{amsthm}
\usepackage{thm-restate}
\usepackage[ruled]{algorithm2e}
\usepackage{tikz}
\usetikzlibrary{positioning, backgrounds, fit}

\declaretheorem[name=Theorem, numberwithin=section]{theorem}

\declaretheorem[name=Definition, sibling=theorem]{definition}

\declaretheorem[name=Assumption, sibling=theorem]{assumption}

\newcommand{\calx}{\mathcal{X}}

\newcommand{\caln}{\mathcal{N}}
\newcommand{\calm}{\mathcal{M}}

\newcommand{\R}{\mathbb{R}}
\newcommand{\xsyn}{X^{Syn}}
\newcommand{\xreal}{X}
\newcommand{\xpred}{X^*}
\newcommand{\sigmadp}{\sigma_{DP}}
\newcommand{\dx}{\mathrm{d}}
\newcommand{\med}{\mathrm{MED}}
\newcommand{\medl}{\med_\theta}
\newcommand{\mul}{\mu(\theta)}
\newcommand{\Sigmal}{\Sigma(\theta)}
\newcommand{\privs}{\tilde{s}}
\newcommand{\privu}{\tilde{u}}

\newcommand{\toydatadelta}{= 2.5\cdot 10^{-7}}
\newcommand{\adultdelta}{\approx 4.7\cdot 10^{-10}}
\newcommand{\uscensusdelta}{\approx 9.7\cdot 10^{-12}}

\makeatletter
\DeclareRobustCommand{\rvdots}{%
  \vbox{
    \baselineskip4\p@\lineskiplimit\z@
    \kern-\p@
    \hbox{.}\hbox{.}\hbox{.}
}}
\makeatother

\title{Noise-Aware Statistical Inference with Differentially Private Synthetic Data}

\makeatletter
\let\inserttitle\@title
\makeatother

\begin{document}

\runningauthor{Ossi Räisä,\, Joonas Jälkö,\, Samuel Kaski,\, Antti Honkela}

\renewcommand{\thefootnote}{\fnsymbol{footnote}}
\twocolumn[

\aistatstitle{Noise-Aware Statistical Inference with Differentially Private Synthetic Data}

\aistatsauthor{ 
  Ossi Räisä \\ University of Helsinki \\ \texttt{ossi.raisa@helsinki.fi}
  \And Joonas Jälkö \\ University of Helsinki\footnotemark[1] \\ \texttt{joonas.jalko@helsinki.fi}
  \AND Samuel Kaski \\ Aalto University, University of Manchester \\ \texttt{samuel.kaski@aalto.fi}
  \And Antti Honkela \\ University of Helsinki \\ \texttt{antti.honkela@helsinki.fi}
}

\aistatsaddress{ 
} 
]

\begin{abstract}
  While generation of synthetic data under differential privacy (DP)
  has received a lot of attention in the data privacy community, 
  analysis of synthetic data has received much less.
  Existing work has 
  shown that simply analysing DP synthetic data as if it were real does not 
  produce valid 
  inferences of population-level quantities. For example, confidence intervals 
  become too narrow, which we demonstrate with a simple experiment.
  We tackle this problem by combining synthetic data analysis 
  techniques from the field of multiple imputation (MI), and
  synthetic data generation using noise-aware (NA) Bayesian modeling into 
  a pipeline NA+MI that allows computing accurate uncertainty estimates for 
  population-level quantities from DP synthetic data.
  To implement NA+MI for discrete data generation using the values of 
  marginal queries, we 
  develop a novel noise-aware synthetic data generation algorithm 
  NAPSU-MQ using the principle of maximum entropy. Our experiments demonstrate 
  that the pipeline is able to produce accurate confidence intervals from DP 
  synthetic data. The intervals become wider with tighter privacy to 
  accurately capture the additional uncertainty stemming from DP noise.
\end{abstract}

\footnotetext[1]{Work done while at Aalto University.}
\renewcommand{\thefootnote}{\arabic{footnote}}

\section{INTRODUCTION}
Availability of data for research is constrained by the dilemma between 
privacy preservation and potential gains obtained from sharing. As a result,
many datasets are kept confidential to mitigate the possibility of 
privacy violations, with access only granted to researchers after a lengthy
approval process, if at all, slowing down research.

One approach to solving the dilemma between free access and confidentiality 
is releasing synthetic data, as proposed by \textcite{rubin1993statistical}.
The idea is that the \emph{data holder} releases a synthetic dataset that 
is based on a real dataset. \emph{Data analysts} can use the synthetic 
dataset instead of the real one for their \emph{downstream analysis}.

The synthetic dataset should maintain population-level statistical 
properties of the original, which are of 
interest to the analysts. Privacy-protection of the synthetic data 
can be guaranteed by employing 
differential privacy (DP)~\autocite{dworkCalibratingNoiseSensitivity2006}, 
which offers provable protection, unlike non-DP synthetic data generation 
methods. 

The analysts of synthetic data should be able to draw valid conclusions 
on the data generating process (DGP) of the real data using the synthetic 
data. An important component of the conclusion in any scientific research 
is estimation of uncertainty, usually in the form of a confidence interval 
or $p$-value. However, as \textcite{wildeFoundationsBayesianLearning2021}
point out, simply using synthetic data as if it were real data only 
allows drawing conclusions about the synthetic data generating process, 
not the real DGP.

\begin{figure*}
  \centering
  \includegraphics[width=\linewidth]{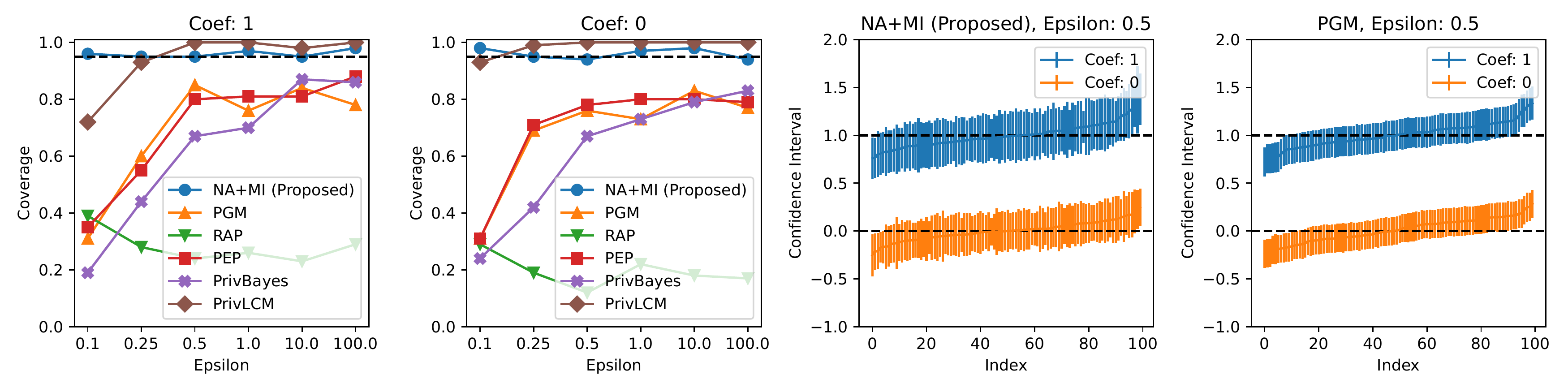}
  \caption{
    Toy data experiment results of logistic regression on 3 binary variables,
    showing that PGM, PEP, RAP and PrivBayes are overconfident, 
    even with almost no privacy ($\epsilon = 100$), PrivLCM is underconfident,
    while our algorithm is well-calibrated.
    The first two panels from the left 
    show the fraction of the 95\% confidence intervals that contain the true 
    parameter value in 100 repeats, with the target confidence level of 95\%
    highlighted by a black line. $\epsilon$ quantifies the strength of the 
    privacy guarantee, ranging from strong ($\epsilon = 0.1$) to meaningless 
    ($\epsilon = 100$). The other privacy parameter is fixed at 
    $\delta \toydatadelta$. The third panel shows the confidence 
    intervals from synthetic data generated by 
    our mechanism for $\epsilon = 0.5$, and the fourth panel shows the confidence 
    intervals from PGM. The third and fourth panels show that the overconfidence 
    stems from intervals that are too narrow, a result of failing to 
    account for all uncertainty.
  }
  \label{fig:toy-data}
\end{figure*}

The issues of using synthetic data in place of real data are especially 
apparent with DP, as DP requires adding noise to the synthetic data generation 
process. We illustrate this with a simple toy data experiment.
We generate 3-dimensional binary data, where one variable is generated 
from logistic regression on the other two, serving as the original dataset. 
Then, we generate synthetic data from the original data and compute confidence 
intervals for the coefficients of the logistic regression from the synthetic 
data. A more detailed description of the setup is given in 
Section~\ref{sec:toy-data-experiment}. As shown by Figure~\ref{fig:toy-data},
the existing algorithms 
PGM~\autocite{mckennaGraphicalmodelBasedEstimation2019},
PEP~\autocite{liuIterativeMethodsPrivate2021}, 
RAP~\autocite{aydoreDifferentiallyPrivateQuery2021}
and PrivBayes~\autocite{zhangPrivBayesPrivateData2017}
do not provide valid confidence intervals, as they treat the synthetic data 
as real data. Only our method NA+MI and 
PrivLCM~\autocite{nixonLatentClassModeling2022} produce valid confidence 
intervals. However, PrivLCM is too conservative, producing much wider intervals,
as shown by Figure~\ref{fig:toy-data-widths}, and does not scale to complex 
datasets.

Our solution to overconfident uncertainty estimates builds on
Rubin's original work on synthetic data generation~\autocite{rubin1993statistical}. 
He proposed generating multiple 
synthetic datasets, running the analysis task on each of them, and combining
the results with simple combining rules called \emph{Rubin's 
rules}~\autocite{raghunathanMultipleImputationStatistical2003,reiter2002satisfying}. 
This workflow is modeled after 
\emph{multiple imputation}~\autocite{rubinMultipleImputationNonresponse1987}, 
where it is used to deal with missing data.
Generating multiple 
synthetic datasets allows the combining rules to account for the additional 
uncertainty that comes from the synthetic data generation process, which includes 
DP noise when the synthetic data is generated using a Bayesian model that 
accounts for the DP noise. We call the combined pipeline of noise-aware (NA) 
private synthetic data generation 
and analysis with multiple imputation (MI) the \emph{NA+MI pipeline}.
We give a more detailed description in Section~\ref{sec:na+mi}.

To implement the NA step, we develop an algorithm called 
\emph{Noise-Aware Private Synthetic data Using Marginal Queries} (NAPSU-MQ),
that generates synthetic data from discrete tabular datasets using the 
noisy values of preselected marginal queries. We describe NAPSU-MQ in 
Section~\ref{sec:generator}. In Section~\ref{sec:experiments} and Supplemental
Section~\ref{app:us-census-experiment},
we evaluate NAPSU-MQ on the UCI Adult and UCI US Census (1990) datasets, 
showing that it can produce accurate confidence intervals. 

\subsection{Related Work}
There is a sizable literature on DP synthetic data generation. Most recent work 
in the area either releases the values of a set of simple queries, such as counting 
queries, under DP and uses them as the basis of synthetic data~\autocite{
  hardtSimplePracticalAlgorithm2012, chenDifferentiallyPrivateHighDimensional2015,
  zhangPrivBayesPrivateData2017, mckennaGraphicalmodelBasedEstimation2019,
  mckennaWinningNISTContest2021a, mckennaAIMAdaptiveIterative2022,
  mckenna2018optimizing, bernsteinDifferentiallyPrivateLearning2017,
  cai2021data, vietriNewOracleEfficientAlgorithms2020, 
  liuIterativeMethodsPrivate2021,
  aydoreDifferentiallyPrivateQuery2021, nixonLatentClassModeling2022,
}, 
or trains some kind of generative model,
often a GAN, using the whole real dataset under DP~\autocite{
  xieDifferentiallyPrivateGenerative2018,
  yoon2018pategan, chenGSWGANGradientsanitizedApproach2020,
  longGPATEScalableDifferentially2021, 
  jalkoPrivacypreservingDataSharing2021,
}.
There are also hybrid approaches that use sophisticated queries that can 
capture all features of the dataset, and train a generative model 
using 
those~\autocite{harderDPMERFDifferentiallyPrivate2021, liewPEARLDataSynthesis2022}.
Of the existing DP synthetic data generation algorithms, NAPSU-MQ is closest 
to the PGM algorithm~\autocite{mckennaGraphicalmodelBasedEstimation2019}, which does 
maximum likelihood estimation with the same data model as NAPSU-MQ instead of 
noise-aware Bayesian inference. We describe this connection in more detail 
in Supplemental Section~\ref{app:napsu-mq-vs-pgm}.

Rubin's rules were originally developed for analyses on missing data, as 
part of an approach called 
multiple imputation~\autocite{rubinMultipleImputationNonresponse1987},
which was later applied to generate and analyse synthetic 
data~\autocite{rubin1993statistical} without DP. The variant of Rubin's rules 
that we use, and describe in Supplemental Section~\ref{app:mi}, was developed 
specifically for synthetic data 
generation~\autocite{raghunathanMultipleImputationStatistical2003, reiter2002satisfying}.
\textcite{raabPracticalDataSynthesis2018} have developed simpler alternatives 
to Rubin's rules under more restrictive assumptions, but these assumptions 
rule out DP data synthesisers.

Rubin's rules have not been widely used with DP synthetic data generation,
and we are only aware of four existing works studying the combination.
\textcite{charestHowCanWe2010} studied Rubin's rules with a very simple 
early synthetic data generation algorithm, and concluded that Rubin's rules 
are not appropriate for that algorithm. 
\textcite{zhengDifferentialPrivacyBayesian2015} found that some 
simple one-dimensional methods developed by the multiple imputation community 
are in fact DP, but not with practically useful privacy bounds.
\textcite{nixonLatentClassModeling2022} propose using Rubin's rules 
with the noise-aware synthetic data generation algorithm PrivLCM, but they only 
consider computing confidence intervals of query values on the real dataset,
not confidence intervals of population parameters of arbitrary downstream analyses.
\textcite{liuModelbasedDifferentiallyPrivate2022} proposes generating multiple 
synthetic datasets like we do, but their pipeline requires splitting the privacy 
budget between each synthetic dataset, severely limiting the number of datasets 
that can be generated with acceptable utility, and their convergence 
theory assumes weak privacy\footnote{The theory applies asymptotically 
when $\epsilon \to \infty$.}.

Noise-aware uncertainty estimates have been 
developed for specific DP analyses. Examples include frequentist linear 
regression~\autocite{evansStatisticallyValidInferences2022} and 
general recipes for DP analyses without synthetic 
data~\autocite{ferrandoParametricBootstrapDifferentially2022,
covingtonUnbiasedStatisticalEstimation2021}.
Bayesian examples include posterior inference for simple exponential family 
models~\autocite{bernsteinDifferentiallyPrivateBayesian2018},
linear regression~\autocite{bernsteinDifferentiallyPrivateBayesian2019}, 
generalised linear models~\autocite{kulkarniDifferentiallyPrivateBayesian2021}, and 
approximate Bayesian computation~\autocite{gongExactInferenceApproximate2022}.
The AIM algorithm for synthetic data 
generation~\autocite{mckennaAIMAdaptiveIterative2022} provides valid confidence 
intervals on the query values of the real dataset, while NAPSU-MQ and the other 
methods mentioned provide confidence intervals on population values.

Several works study techniques for mitigating the effect of DP noise.
\textcite{wildeFoundationsBayesianLearning2021} point out the importance 
of noise-aware synthetic data analysis with DP and use publicly available data to 
augment the analysis and correct for the DP noise in Bayesian inference.
Other examples include bias 
reduction~\autocite{ghalebikesabiBiasMitigatedLearning2021} and averaging 
GANs~\autocite{neunhoefferPrivatePostGANBoosting2021}.

While some of the existing works address uncertainty estimates for specific 
analyses of synthetic data under DP, 
there is no existing method for proper 
uncertainty estimation for general downstream analyses of population-level 
quantities. We fill this gap with the NA+MI pipeline, which we implement for
discrete tabular data with NAPSU-MQ.

\section{THE NA+MI PIPELINE}\label{sec:na+mi}
The early work on synthetic data generation with multiple imputation showed 
that computing accurate uncertainty estimates with synthetic data 
requires accounting for the additional uncertainty that comes from generating the
synthetic 
data~\autocite{rubin1993statistical,raghunathanMultipleImputationStatistical2003}.
\textcite{rubin1993statistical} proposed generating multiple synthetic datasets 
$\xsyn_i$ from the posterior predictive distribution 
$p(\xpred|\xreal)$, where $\xpred$ 
is a prediction of a future dataset, and $\xreal$ is the observed real dataset. 
The downstream analysis is run on each $\xsyn_i$ as if $\xsyn_i$ were the real 
dataset, and the results are combined using specialised combining 
rules~\autocite{raghunathanMultipleImputationStatistical2003}.

The generation of multiple datasets from $p(\xpred|\xreal)$ is necessary to give the 
combining rules a way to estimate the variance of the synthetic data generation 
process, which would not be possible if only a single dataset was generated.
For a parametric model, 
$p(\xpred|\xreal) = \int_\Theta p(\xpred|\theta)p(\theta | \xreal) \dx\theta$, where 
$\theta\in \Theta$ is the parameter, $p(\xpred|\theta)$ is the likelihood, 
and $p(\theta | \xreal)$ is the posterior of the parameter. 
$\xsyn_i$ is then generated in two steps: first $\theta_i\sim p(\theta|\xreal)$ 
is sampled, then $\xsyn_i\sim p(\xpred|\theta_i)$. 

The combining rules require including the posterior distribution 
$p(\theta|X)$~\autocite{raghunathanMultipleImputationStatistical2003} for sampling 
$\xsyn_i$,
so a non-Bayesian model that samples $\xsyn_i\sim p(\xpred|\hat{\theta})$
for some point-estimate $\hat{\theta}$ is not suitable for synthetic 
data generation with Rubin's rules.

Requiring the synthetic data generation to be DP complicates 
the picture, as only a noisy observation $\privs$ of $\xreal$ can be made,
which means that we must use $p(\theta | \privs)$ instead of $p(\theta | \xreal)$.
We call inference algorithms for $p(\theta | \privs)$ that account for the noise 
added for DP \emph{noise-aware}.
The combination of noise-aware inference and multiple imputation is the 
NA+MI pipeline, which we summarise in Figure~\ref{fig:mech-summary}.
First, the data holder runs inference on a noise-aware Bayesian model 
using the private data, which we call the NA step. 
Different implementations of the NA step may set different requirements on 
the form of $\xreal$, the type of DP noise, and may provide different 
privacy guarantees.

\begin{figure*}
  \center
  \begin{tikzpicture}[
    ellipsis/.style={rectangle, align=center, minimum size=1mm},
    ownernode/.style={rectangle, draw=black!40!red, very thick, align=center},
    analystnode/.style={rectangle, draw=black!70, very thick, align=center},
    invisnode/.style={rectangle, align=center},
    node distance=7mm,
    ]
    \node[ownernode] (data) 
    {Data $X$};
    \node[analystnode] (generator) [right=of data] {Noise-aware\\ Generator};

    \node[ellipsis] (syndatadots) [right=1.5cm of generator] {$\rvdots$};
    \node[analystnode] (syndatatop) [above=1mm of syndatadots] {Synthetic\\ Data $\xsyn_1$};
    \node[analystnode] (syndatabot) [below=1mm of syndatadots] {Synthetic\\ Data $\xsyn_m$};

    \node[ellipsis] (analysisdots) [right=2.4cm of syndatadots] {$\rvdots$};
    \node[analystnode, minimum width=23mm] (analysistop) [right=of syndatatop] {Analysis\\ Result $q_1, v_1$};
    \node[analystnode, minimum width=23mm] (analysisbot) [right=of syndatabot] {Analysis\\ Result $q_m, v_m$};

    \node[analystnode] (result) [right=of analysistop, yshift=-9mm] {Combined\\ Result\\ $t_\nu(\bar{q}, T^*)$};

    \node[invisnode] (privbarriertop) [right=of data, xshift=-4mm, yshift=14mm] {};
    \node[invisnode] (privbarrierbot) [right=of data, xshift=-4mm, yshift=-14mm] {Privacy Barrier};
    
    \draw[->, thick] (data) -- (generator);
    \draw[->, thick] (generator) -- (syndatatop);
    \draw[->, thick] (generator) -- (syndatabot);
    \draw[->, thick] (syndatatop) -- (analysistop);
    \draw[->, thick] (syndatabot) -- (analysisbot);
    \draw[->, thick] (analysistop) -- (result);
    \draw[->, thick] (analysisbot) -- (result);

    \draw[dashed, very thick] (privbarriertop.west) -- (privbarrierbot.west);

    \begin{scope}[on background layer]
      \node [
        rounded corners=3mm, fill=blue!160, opacity=0.1, 
        fit=(data) (generator) (syndatatop) (syndatabot), 
        label=above:Data Holder (NA step)
      ] {};
      \node [
        rounded corners=3mm, fill=orange!150, opacity=0.1, 
        fit=(analysistop) (analysisbot) (result) (syndatatop) (syndatabot), 
        label=above:Data Analyst (MI step)
      ] {};
    \end{scope}
  \end{tikzpicture}
  \caption{
    NA+MI pipeline for noise-aware DP synthetic data generation and 
    statistical inference. The nodes shaded in blue are computed by the data holder,
    and the nodes shaded in orange are computed by the data analyst. All nodes 
    except Data (with red border) can be released to the public. The 
    synthetic datasets can be generated by either party because the Generator is 
    also released by the data holder.
  }
  \label{fig:mech-summary}
\end{figure*}
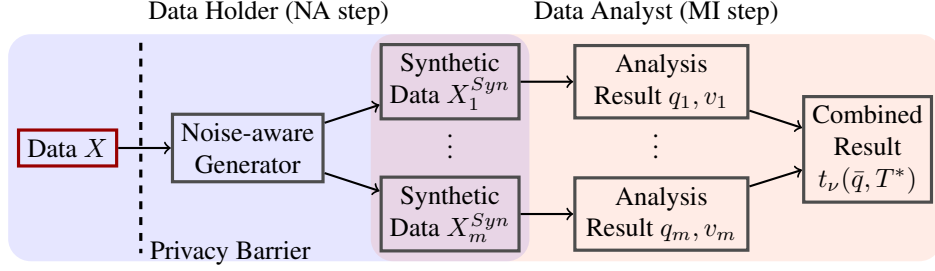

After the inference, the data holder generates multiple synthetic 
datasets.
The data holder can also release the posterior distribution in addition to 
the synthetic datasets, so that the analyst can also generate synthetic 
datasets if needed. Due to the post-processing immunity of DP 
(Theorem~\ref{thm:post-processing-immunity}), releasing multiple synthetic 
datasets, or the posterior distribution, does not compromise privacy.

For each synthetic dataset, the analyst runs their analysis, and combines 
the results using multiple 
imputation~\autocite{
  raghunathanMultipleImputationStatistical2003, reiter2002satisfying,
  reiter2005significance,
}. 
We call this the MI step.
For frequentist downstream analyses, we use Rubin's rules, which
require that each analysis produces a point estimate $q_i$, and a variance 
estimate $v_i$ for the point estimate. The point estimates $q_1, \dotsc, q_m$ and 
the variance estimates $v_1,\dotsc, v_m$ are fed to Rubin's 
rules~\autocite{raghunathanMultipleImputationStatistical2003}, which give 
a $t$-distribution for the estimated quantity that the analyst can use 
to compute confidence intervals or 
hypothesis tests. We describe Rubin's rules in more detail in Supplemental 
Section~\ref{app:mi}.

\section{BACKGROUND FOR NAPSU-MQ}\label{sec:background}
In this section we describe the datasets and queries NAPSU-MQ uses, and 
briefly describe key concepts from differential privacy,
which we will use in Section~\ref{sec:generator}. 

\paragraph{Data and Marginal Queries}
NAPSU-MQ uses tabular datasets of $d$ discrete variables, where the domains of 
the variables, and the number datapoints $n$ are known.
We denote the set of possible datapoints by $\calx$, and the set of 
possible datasets by $\calx^n$. For a set of variables $I$ and $x\in \calx$, 
$x^{(I)}$ denotes the selection of the variables in $I$ from $x$.

\begin{definition}
  A marginal query of variables $I$ and value $v$ is a function 
  $a\colon \calx\to \{0, 1\}$ that takes a datapoint $x$ as input and returns 1 
  if $x^{(I)} = v$ and 0 otherwise.
  For a dataset $X\in \calx^n$, we define $a(X) = \sum_{i=1}^n a(x_i)$,
  where $x_i$ is the $i$:th datapoint in $X$.
\end{definition}
When $I$ has $k$ variables, $a$ is called a $k$-way marginal query.

When evaluating multiple marginal queries $a_1, \dotsc, a_{n_q}$, we 
concatenate their values to a vector-valued function 
$a\colon \calx\to \{0, 1\}^{n_q}$.
We call the concatenation of marginal queries for all possible values of variables 
$I$ the \emph{full set of marginals on $I$}~\footnote{
  Some existing works~\autocite{mckennaGraphicalmodelBasedEstimation2019} 
  use the term marginal query for the full set of marginal queries.
  We chose this terminology because we deal with individual marginal queries 
  in Supplemental Section~\ref{sec:canon-queries}.
}.

As a concrete example, take the 3-dimensional binary data used in 
Figure~\ref{fig:toy-data}. A single 2-way marginal query could, for example, 
look at the first two variables of a datapoint, and check that both are 0.
The full set of marginal queries on the first two variables checks 
which of the 4 possible values the first two variables of a datapoint have, 
and returns a 4-component vector with a single 1 and 3 zeros, indicating 
the answer. As input to NAPSU-MQ, we could use the concatenation of the 
3 full sets of marginal queries for each pair of variables\footnote{
  Our experiments for the toy dataset use the full set of 3-way marginals 
  that includes all 3 variables.
}.
Including these kinds of marginals with more than one variable allows NAPSU-MQ 
to take the dependencies between variables into account.

\paragraph{Differential Privacy}
Differential privacy 
(DP)~\autocite{dworkCalibratingNoiseSensitivity2006, dworkOurDataOurselves2006}
is a definition aiming to quantify the privacy loss resulting from releasing 
the results of some algorithm. DP algorithms are also called \emph{mechanisms}.

\begin{definition}
  A mechanism $\calm$ is $(\epsilon, \delta)$-differentially private if 
  for all neighbouring datasets $X, X'$ and all measurable output sets $S$
  \begin{equation}
    P(\calm(X)\in S) \leq e^\epsilon P(\calm(X')\in S) + \delta.
  \end{equation}
\end{definition}
The neighbourhood relation in the definition is domain 
specific. We use
the \emph{substitute} neighbourhood relation for tabular datasets, 
where datasets are neighbouring if they differ in at most one datapoint.

Together, $\epsilon$ and $\delta$ bound the tradeoff between false positive and 
false negative rates for any hypothesis 
test~\autocite{kairouzCompositionTheoremDifferential2015}. Their choice is a 
matter of policy~\autocite{dworkDifferentialPrivacySurvey2008}.
$\delta \approx \frac{1}{n}$ permits mechanisms that clearly violate 
privacy~\autocite{dworkAlgorithmicFoundationsDifferential2014}, so one should 
choose $\delta \ll \frac{1}{n}$.

The mechanism we use to release marginal query values under DP 
is the \emph{Gaussian mechanism}~\autocite{
  dworkOurDataOurselves2006, balleImprovingGaussianMechanism2018,
}.
\begin{definition}
  The Gaussian mechanism with noise variance $\sigmadp^2$ adds Gaussian
  noise to the value of a function $f\colon \calx^n\to \R^{k}$ for input data $X$:
  $\calm(X) = f(X) + \caln(0, \sigmadp^2 I)$.
\end{definition}

The privacy bounds of the Gaussian mechanism depend on the \emph{sensitivity}
of the function $f$, which is an upper bound on the change in the 
value of $f$ for neighbouring datasets. Larger sensitivities require a larger 
noise variance.
\begin{definition}
  The $L_2$-sensitivity of a function $f$ is 
  $\Delta_2 f = \sup_{X\sim X'} ||f(X) - f(X')||_2$.
  $X\sim X'$ denotes that $X$ and $X'$ are neighbouring.
\end{definition}

The sensitivity of a concatenation of full sets of marginal queries has a 
simple form:
\begin{restatable}{theorem}{theoremmarginalsensitivity}\label{thm:marginal-sensitivity}
  Let $a$ be the concatenation of $n_s$ full sets of marginal queries.
  Then $\Delta_2 a \leq \sqrt{2n_s}$.
\end{restatable}
\begin{proof}
  We defer the proof to Supplemental Section~\ref{app:proofs}.
\end{proof}

\begin{theorem}[\textcite{balleImprovingGaussianMechanism2018}]\label{thm:gauss-mech-dp}
  The Gaussian mechanism for a function $f$ with $L_2$-sensitivity 
  $\Delta_2$ and noise variance $\sigmadp^2$ is $(\epsilon, \delta)$-DP with 
  \begin{equation}
    \delta \geq \Phi\left(\frac{\Delta_2}{2\sigmadp} 
    - \frac{\epsilon \sigmadp}{\Delta_2}\right) 
    - e^\epsilon\Phi\left(\frac{-\Delta_2}{2\sigmadp} 
    - \frac{\epsilon \sigmadp}{\Delta_2}\right)
  \end{equation}
  where $\Phi$ is the cumulative distribution function of the standard Gaussian 
  distribution.
\end{theorem}
If $\epsilon$, $\delta$ and $\Delta_2$ are given, which is typical, 
$\sigmadp^2$ can be solved from Theorem~\ref{thm:gauss-mech-dp} using 
standard numerical methods.

An important property of DP is post-processing immunity: post-processing 
the result of a DP-algorithm does not weaken the privacy bounds.
\begin{theorem}[\textcite{dworkAlgorithmicFoundationsDifferential2014}]\label{thm:post-processing-immunity}
  Let $\calm$ be an $(\epsilon, \delta)$-DP mechanism, and let 
  $f$ be any algorithm. Then the composition $f\circ \calm$ is 
  $(\epsilon, \delta)$-DP. 
\end{theorem}

Because of post-processing immunity, we can release marginal query values 
with the Gaussian mechanism to obtain privacy bounds, and make arbitrary use of 
the noisy query values without weakening the privacy bounds. This allows 
releasing an arbitrary number of synthetic datasets without compromising on 
privacy.

\section{NOISE-AWARE SYNTHETIC DATA GENERATION}\label{sec:generator}
In order to implement the NA step, the data holder needs 
to generate synthetic data from the posterior of a noise-aware Bayesian model. 
\textcite{bernsteinDifferentiallyPrivateBayesian2018} develop 
noise-aware Bayesian inference under DP for simple exponential family models,
with noise added to sufficient statistics.
However, their algorithm requires both computing unnormalised densities 
for the model, and sampling from the model's conjugate prior. In our setting,
neither of these is trivial.

We implement the NA step by generalising their algorithm to use arbitrary marginal 
queries as the sufficient statistics, using the maximum entropy 
distribution of the marginal queries as our 
exponential family model. We develop a computationally feasible method to 
compute the unnormalised density of this model, and sidestep the conjugate 
prior sampling by using standard posterior sampling methods, which also allows 
us to use a non-conjugate prior.

We start by considering an arbitrary set of marginal queries $a$.
We would like to find an exponential family distribution that, in expectation, 
gives the same answers to $a$ as the real data $\xreal$. We do not want to 
assume anything else about the distribution besides these expected query 
values, so we use the principle of maximum 
entropy~\autocite{jaynesInformationTheoryStatistical1957} to choose the distribution.

The distribution with maximal entropy that satisfies the constraint 
$E_{\xpred\sim P}(a(\xpred)) = a(X)$ is
\begin{equation}
  P(x) = \exp(\theta^T a(x) - \theta_0(\theta))
\end{equation}
for some parameters 
$\theta\in \R^k$~\autocite{wainwrightGraphicalModelsExponential2008}, where 
$k$ is the number of queries.
$\theta_0(\theta)$ is the normalising constant of the distribution, so 
it is given by 
\begin{equation}
  \theta_0(\theta) = \ln\left(\sum_{x\in \calx}\exp(\theta^T a(x))\right).
\end{equation}
We denote this distribution by $\medl$, and use $\medl^n$ to denote 
the distribution of $n$ i.i.d. samples from $\medl$.

$\medl$ is clearly an exponential family distribution, with sufficient 
statistic $a(x)$ and natural parameters $\theta$.
$\medl$ is also a Markov network~\autocite{koller2009probabilistic},
given in log-linear form.

The Bayesian model we consider is derived from the generative process of 
the noisy query values, which are observed. Assuming that 
the data generating process is $\medl$,
and knowing that the Gaussian mechanism adds noise with variance 
$\sigma_{DP}^2$, we get the probabilistic model
\begin{align}
  \theta &\sim \mathrm{Prior}, \quad
  &\xreal &\sim \medl^n, \quad
  \\s &= a(\xreal), \quad
  &\privs &\sim \caln(s, \sigmadp^2I).
\end{align}
In principle, we could now sample from the posterior $p(\theta\mid \privs)$, 
with $s$ marginalised out. In practice, the marginalisation is not feasible, 
as $s$ is a discrete variable with a very large domain. 

However, $s$ is a sum of the query values for individual data\-points, 
so asymptotically $s$ has a normal distribution by the 
central limit theorem. We can substitute the normal approximation for 
$s$ into the model, which allows us to easily marginalise $s$ out, resulting in
\begin{align}
  \theta &\sim \mathrm{Prior}, \quad
  \privs \sim \caln(n\mul, n\Sigmal + \sigmadp^2I),
\end{align}
where $\mul$ and $\Sigmal$ denote the mean and covariance 
of $a(x)$ for a single sample $x\sim \medl$.

To compute $\mul$ and $\Sigmal$, we use both the exponential 
family and Markov network structure of $\medl$. Due to the exponential family 
structure,
\begin{align}
  \mul = \nabla \theta_0(\theta), \quad \quad
  \Sigmal = H_{\theta_0}(\theta),
\end{align}
where $H_{\theta_0}$ is the Hessian of $\theta_0$.
Computing $\theta_0$ naively requires summing over the exponentially large 
domain $\calx$, which is not tractable for complex domains. The Markov network 
structure gives a solution: $\theta_0$ can be computed with the variable 
elimination algorithm~\autocite{koller2009probabilistic}. We can then autodifferentiate 
variable elimination to compute $\mul$ and $\Sigmal$. Alternatively, 
$\mul$ can be computed by belief propagation~\autocite{koller2009probabilistic}, 
and $\Sigmal$ can be 
autodifferentiated from it. 

For the prior, we choose another Gaussian distribution with mean 0 and 
standard deviation 
10, which is a simple and weak prior, but other priors could be used.

To sample the posterior, we use existing posterior inference methods, 
namely the Laplace approximation~\autocite{gelmanBayesianDataAnalysis2014},
which approximates the posterior with a Gaussian distribution centered 
at the posterior mode, and 
the NUTS algorithm~\autocite{hoffmanNoUTurnSamplerAdaptively2014}, which 
is a \emph{Markov chain Monte Carlo} (MCMC) algorithm that samples the 
posterior directly using the gradients of the posterior log-density.

The time complexities of computing $\mul$ and $\Sigmal$ for inference, as well 
as sampling $\medl$ after inference,
depend on the sparsity\footnote{In a sparse graph, the fraction of pairs of nodes with an edge 
between them is small.} of the Markov network graph that the selected set of 
queries induces. Specifically, the time complexities are exponential in the 
\emph{tree width} of the graph~\autocite{koller2009probabilistic}, which 
can be much lower than the dimensionality for sparse graphs, making 
inference and sampling tractable for sparse queries.

If we include all possible marginal queries from the selected variable sets,
the parametrisation of $\medl$ is not identifiable, as there are 
linear dependencies among the queries~\autocite{koller2009probabilistic}. 
Nonidentifiablity causes NUTS 
sampling to be very slow, so we prune the queries to remove linear dependencies 
while preserving the information in the queries 
using the \emph{canonical parametrisation} for 
$\medl$~\autocite{koller2009probabilistic}.
We give a detailed description of the process in 
Supplemental Section~\ref{sec:canon-queries}.

\subsection{NAPSU-MQ Properties}

We summarise NAPSU-MQ in Algorithm~\ref{alg:napsu-mq}. The privacy bounds for 
NAPSU-MQ follow from the material of Section~\ref{sec:background}:
\begin{theorem}
  NAPSU-MQ (Algorithm~\ref{alg:napsu-mq}) is $(\epsilon, \delta)$-DP with 
  regards to the real data $\xreal$.
\end{theorem}
\begin{proof}
  The returned values of Algorithm~\ref{alg:napsu-mq} only depend on the 
  real data $\xreal$ through $\privs$. Releasing $\privs$ is 
  $(\epsilon, \delta)$-DP due to the selection of $\sigmadp^2$ with 
  Theorem~\ref{thm:gauss-mech-dp}. Computing the returned values from 
  $\privs$ is post-processing, so by Theorem~\ref{thm:post-processing-immunity},
  NAPSU-MQ is $(\epsilon, \delta)$-DP.
\end{proof}

In Supplemental Section~\ref{app:mi}, we list the conditions under which 
Rubin's rules are unbiased~\autocite{raghunathanMultipleImputationStatistical2003}.
As is typical with statistical methods, these assumptions are asymptotic in 
nature. Our experiments in 
Section~\ref{sec:experiments} show that NAPSU-MQ is robust to these 
asymptotics and works with large enough samples.

The most important of the assumptions behind Rubin's rules\footnote{
  Assumptions~\ref{assu:rubin-3} and \ref{assu:rubin-4} in 
  Supplemental Section\ref{app:mi}.
}
for synthetic data generation
requires that the synthetic data generation does not bias the downstream 
analysis. This means that the marginal queries input to NAPSU-MQ must contain 
the relevant information for the downstream analysis. For our experiments, we 
include a fixed set of queries that gives enough information for the downstream 
task, and select other queries automatically with an existing query selection 
algorithm~\autocite{mckennaWinningNISTContest2021a}. We leave further work 
on query selection for noise-aware inference to future work.

\begin{algorithm}
  \caption{NAPSU-MQ}\label{alg:napsu-mq}
  \KwIn{Real data $\xreal$, marginal queries $a$, number of synthetic datasets 
  $m$, size of synthetic datasets $n_{Syn}$, privacy bounds $\epsilon, \delta$.}
  \KwOut{Posterior distribution $p(\theta|\privs)$, synthetic datasets 
  $\xsyn_1, \dotsc, \xsyn_m$.}

  $a^* \gets$ Canonical queries for $a$ (Section~\ref{sec:canon-queries})\;
  $s\gets a^*(\xreal)$\; 
  $\Delta_2 \gets$ Sensitivity of $s$ (Theorem~\ref{thm:marginal-sensitivity})\;
  $\sigmadp^2 \gets$ Required noise variance for $(\epsilon, \delta)$-DP with 
  sensitivity $\Delta_2$ (Theorem~\ref{thm:gauss-mech-dp})\;
  Sample $\privs \sim \caln(s, \sigmadp^2)$\;
  Run Bayesian inference algorithm to find $p(\theta | \privs)$ (Section~\ref{sec:generator})\; 
  Sample $\theta_i\sim p(\theta| \privs)$ and 
  $\xsyn_i\sim \med_{\theta_i}^{n_{Syn}}$
  for $1\leq i\leq m$\; 
  \Return{$p(\theta|\privs)$, $\xsyn_1,\dotsc, \xsyn_m$}
\end{algorithm}

\section{EXPERIMENTS}\label{sec:experiments}
In this section, we give detailed descriptions on our two main experiments:
a simple toy data experiment, and our experiment with the UCI 
Adult dataset, which demonstrate that NAPSU-MQ is able to compute accurate 
confidence intervals. In Supplemental Section~\ref{app:us-census-experiment}, 
we describe an additional experiment with the UCI US Census (1990) dataset,
which confirms the results of the other experiments. 
Our code is available under an open source license\footnote{
  A library implementation of NAPSU-MQ is available at 
  \url{https://github.com/DPBayes/twinify}, while 
  code to replicate our experiments is at 
  \url{https://github.com/DPBayes/NAPSU-MQ-experiments}.
}.

\begin{figure*}
  \centering
  \begin{subfigure}{0.48\textwidth}
    \centering
    \includegraphics[width=\linewidth]{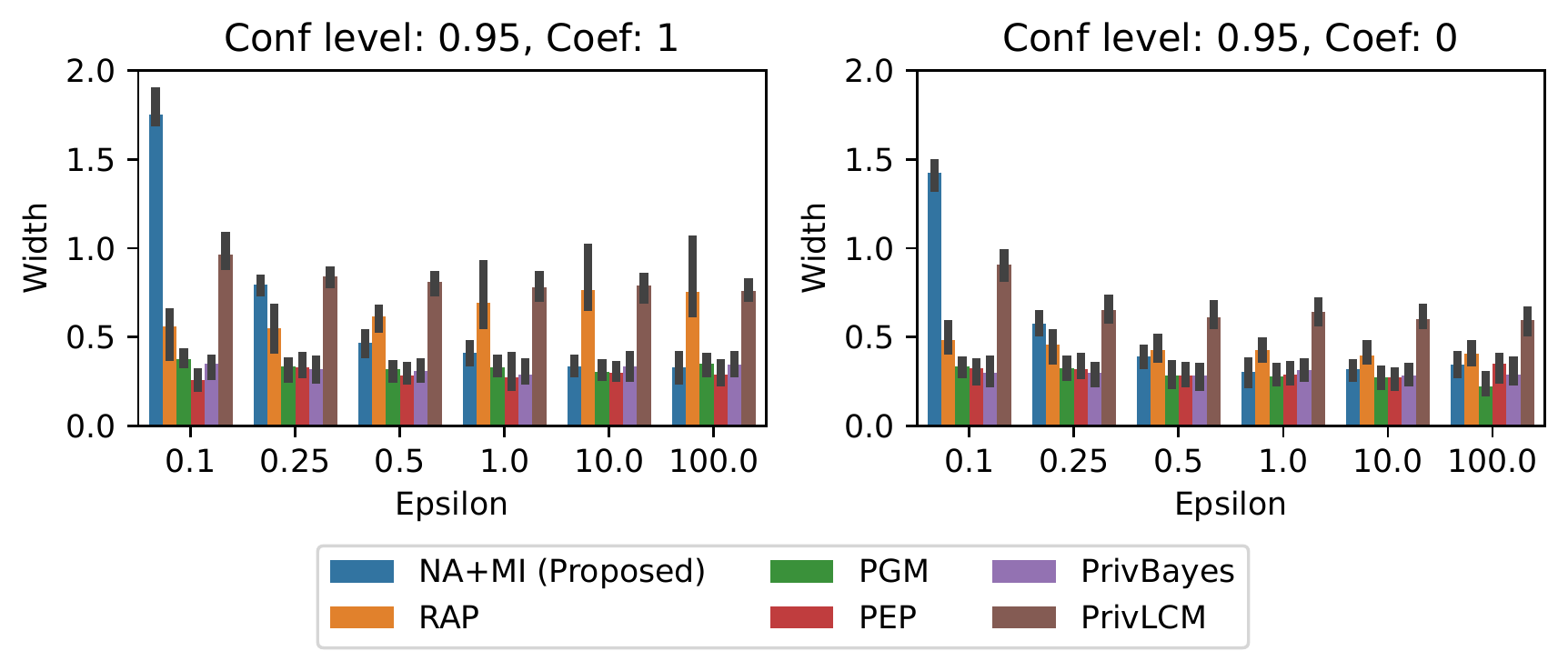}
    \caption{
      }
    \label{fig:toy-data-widths}
  \end{subfigure}\quad
  \begin{subfigure}{0.48\textwidth}
    \centering
    \includegraphics[width=\linewidth]{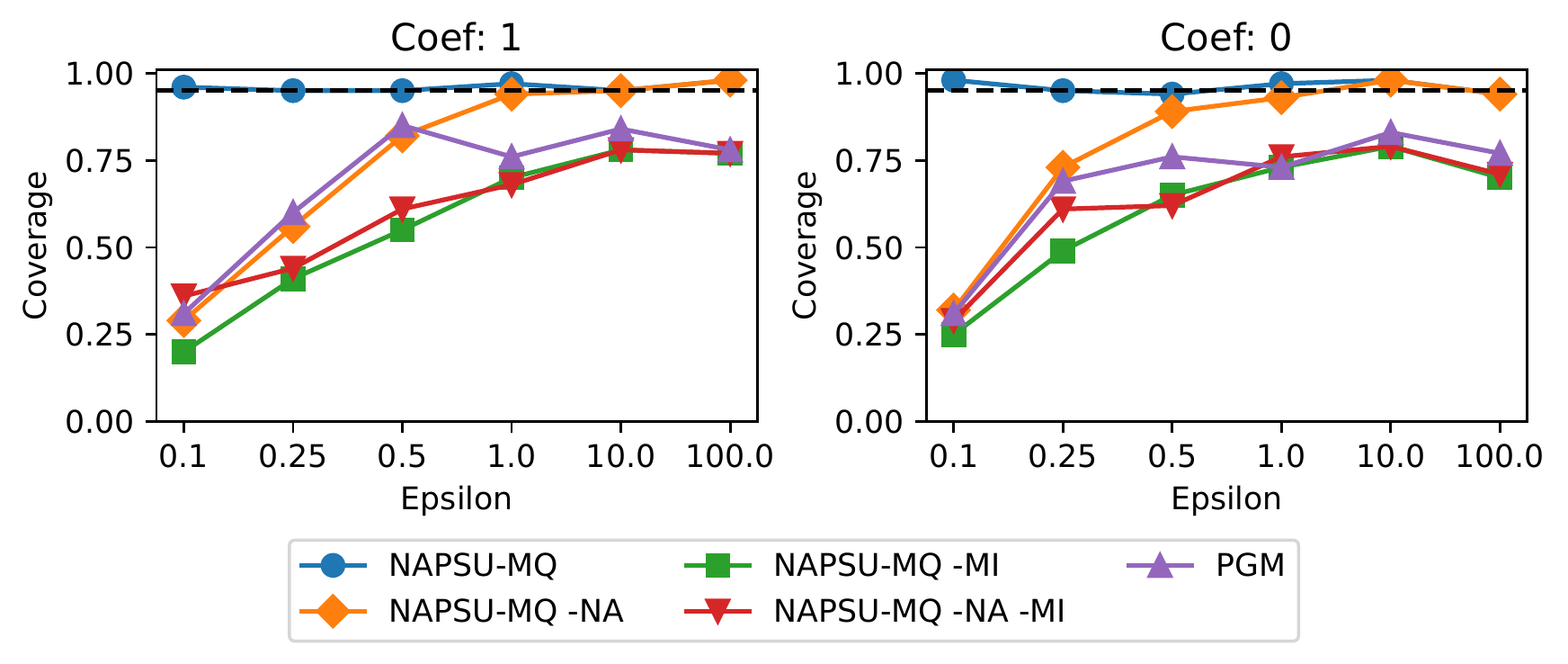}
    \caption{
      }
    \label{fig:ablation}
  \end{subfigure}
  \caption{
      (a) Toy data confidence intervals widths. NA+MI produces slightly wider 
      intervals than PGM or PEP for $\epsilon > 0.1$, which is necessary to 
      account for DP noise. PrivLCM produces much wider intervals.
      (b) Ablation study on the toy data. ``-NA'' refers to removing noise-awareness, 
      and ``-MI'' refers to removing multiple imputation. Unless both are included, 
      NAPSU-MQ is overconfident like PGM except for $\epsilon \geq 1$
      where noise-awareness is not necessary. $\delta \toydatadelta$ in 
      both (a) and (b).
  }
\end{figure*}

\subsection{Toy Data}\label{sec:toy-data-experiment}
To demonstrate the necessity of noise-awareness in synthetic data generation, 
we measure the coverage of confidence intervals computed from DP synthetic 
data on a generated toy dataset where the data generation process is known.
We test the existing algorithms
PGM~\autocite{mckennaGraphicalmodelBasedEstimation2019},
PEP~\autocite{liuIterativeMethodsPrivate2021}, 
RAP~\autocite{aydoreDifferentiallyPrivateQuery2021},
PrivLCM~\autocite{nixonLatentClassModeling2022},
PrivBayes~\autocite{zhangPrivBayesPrivateData2017}
and our pipeline NA+MI, where data generation is implemented with 
NAPSU-MQ. The authors of PrivLCM also propose using multiple 
imputation~\autocite{nixonLatentClassModeling2022}, so we use Rubin's 
rules~\autocite{raghunathanMultipleImputationStatistical2003} with 
the output of PrivLCM. We also ran PGM, PEP, RAP and PrivBayes with Rubin's rules, 
which did not produce useful results, as these algorithms do not meet the 
assumptions of Rubin's rules.

The original data consists of $n = 2000$ datapoints of 3 binary variables. The 
first two are
sampled by independent coin flips. The third is sampled from logistic regression 
on the other two variables with coefficients (1, 0). 

For all algorithms except PrivLCM and PrivBayes, we use the full set of 3-way 
marginal queries released with the Gaussian mechanism. PrivLCM doesn't implement these,
and instead uses all full sets of 
2-way marginals, and a different mechanism, which is 
$(\epsilon, 0)$-DP~\autocite{nixonLatentClassModeling2022} instead of 
$(\epsilon, \delta)$-DP like the other algorithms. PrivBayes requires specifying 
a Bayesian network~\autocite{zhangPrivBayesPrivateData2017}, which we set to 
match the data generating process, and 
takes a single full set of 1-way marginals and a single full set of 
2-way marginals in addition to the full set of 3-way marginals the other 
algorithms take.
We use the Laplace approximation for NAPSU-MQ inference, as it is much faster 
than NUTS and works well for this simple setting.

DP algorithms are typically evaluated under changing privacy bounds by fixing 
$\delta \leq \frac{1}{n}$, and varying 
$\epsilon$, which is the setting used by the authors of 
PGM~\autocite{mckennaGraphicalmodelBasedEstimation2019}, 
PEP~\autocite{liuIterativeMethodsPrivate2021} and
RAP~\autocite{aydoreDifferentiallyPrivateQuery2021}. We follow this setting, 
fixing $\delta = n^{-2} \toydatadelta$.

We generate $m = 100$ synthetic datasets of size 
$n_{Syn} = n$ for all algorithms except RAP, where the synthetic dataset size 
is a function of two hyperparameters. We describe the hyperparameters in detail 
in Supplemental Section~\ref{app:hyperparams}.

The downstream task 
is inferring the logistic regression coefficients from synthetic data. 
We repeated all steps 100 times to measure the probability of 
sampling a dataset giving a confidence interval that includes the true parameter
values. 

Figure~\ref{fig:toy-data} shows the coverages, and 
Figure~\ref{fig:toy-data-widths} shows the widths for the resulting confidence 
intervals. All of the algorithms apart from 
ours and PrivLCM are overconfident, even with very loose privacy bounds. 
Examining the confidence intervals shows the reason:
PGM is unbiased, but it produces too narrow confidence intervals, while 
NAPSU-MQ produces wider confidence intervals. On the other hand, for 
$\epsilon > 0.25$, PrivLCM 
produces much wider and too conservative confidence intervals.

\paragraph{Ablation Study} We also conducted an ablation study on the toy data 
to show that both 
multiple imputation and noise-awareness are necessary for accurate confidence 
intervals. The results are presented in 
Figure~\ref{fig:ablation}. Without both multiple imputation and noise-awareness, 
NA+MI is overconfident like PGM, except for $\epsilon \geq 1$, where 
noise-awareness is not required. 
We show the confidence 
intervals produced by each method for $\epsilon = 0.5$ in
Figure~\ref{fig:ablation-conf-ints} in the Supplement.

\subsection{Adult Dataset}\label{sec:adult-experiment}
Our main experiment evaluates the performance of NAPSU-MQ on the UCI 
Adult
dataset~\autocite{kohaviAdult1996}. We include 10 of the original 15
columns to remove redundant columns and keep runtimes 
manageable, and discretise the continuous columns. After 
dropping rows with missing values, there are $n = 46\,043$ rows. The discretised
domain has $1\,792\,000$ distinct values. We give a detailed description 
of the dataset, query selection and the downstream task in Supplemental 
Section~\ref{app:adult-details}.

\begin{figure*}
  \centering
  \includegraphics[width=\linewidth]{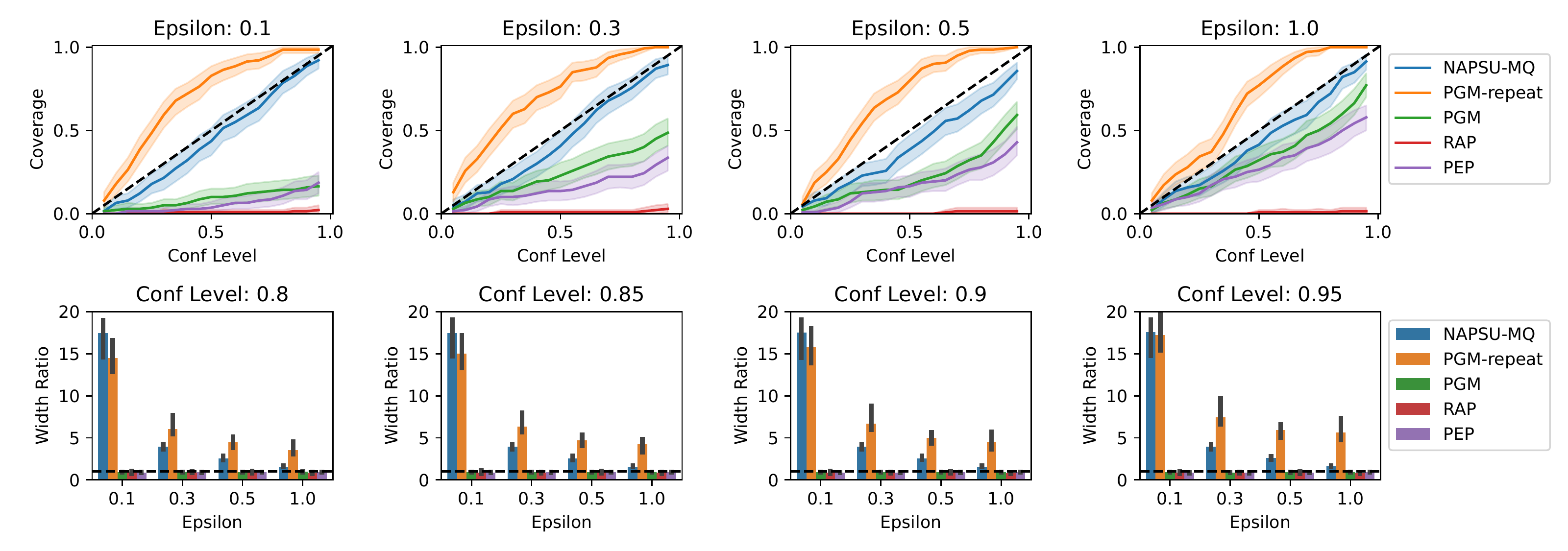}
  \caption{
    Top row: the fraction of downstream coefficients where the synthetic 
    confidence interval contains the real data coefficient, averaged over 
    20 repeated runs on the Adult dataset using regularisation. 
    NAPSU-MQ is the only algorithm that is consistently 
    around the diagonal, showing good calibration. The error bands are
    bootstrap 95\% confidence intervals of the average over the 20 repeats.
    Bottom row: confidence intervals widths divided by real data confidence
    interval widths. Each bar is a median over the different 
    coefficients and repeats, and the black lines are 95\% bootstrapped 
    confidence intervals. The dashed line is at $y = 1$, showing where synthetic 
    confidence intervals have the same width as original confidence intervals.
    $\delta \adultdelta$ in all panels.
  }
  \label{fig:adult-calibration-widths}
\end{figure*}

As our downstream task, we use logistic regression with income as the 
dependent variable and a subset of the columns as the independent variables, 
which 
allows us to include all the relevant marginals for synthetic data generation.
The synthetic dataset was still generated with all 10 columns.

We compare NAPSU-MQ against 
PGM~\autocite{mckennaGraphicalmodelBasedEstimation2019},
RAP~\autocite{aydoreDifferentiallyPrivateQuery2021}
and PEP~\autocite{liuIterativeMethodsPrivate2021}.
We used the published implementations of their authors for all of them, 
with small modifications to ensure compatibility with new library versions and 
our experiments.
The published implementation of PrivLCM only supports binary data, and does 
not scale to datasets of this size, so it was not included in this experiment.
PrivBayes was also excluded, as it doesn't support the set of queries we use
in this experiment.
We also include a naive noise-aware baseline that runs $m$ completely independent 
repeats of PGM, splitting the privacy budget appropriately, and uses 
Rubin's rules with the $m$ generated synthetic datasets. 

NAPSU-MQ and PGM-repeat sometimes generate synthetic datasets with no people 
of some race with high income. Logistic regression will produce an extremely 
wide confidence intervals for the corresponding coefficients. Rubin's rules 
average over estimates, so even a single bad estimate makes the combined 
confidence interval extremely wide. This can be fixed in two ways:
a simple solution is removing estimates with extremely large variances before 
applying Rubin's rules. A more principled way is to add a very small 
regularisation term to the logistic regression, which fixes the extremely 
wide confidence intervals, but will require bootstrapping to get variance 
estimates, which increases the computational cost of the downstream analysis.
We used an $l_2$-regularisation term of $10^{-5}$, and used 50 bootstrap samples.
Because of post-processing immunity, neither of these fixes affects the 
privacy bounds.

As the input queries, we pick 2-way marginals that are relevant for the 
downstream task, and select the rest of the queries with the
MST algorithm~\autocite{mckennaWinningNISTContest2021a}. This selection was kept 
constant throughout the experiment. 
For the privacy budget, we use $\delta = n^{-2} \adultdelta$
for all runs, and vary $\epsilon$. 

\textcite{reiter2002satisfying} discusses the choice of $n_{Syn}$ and 
$n$ for non-DP synthetic data generation in detail. Based on his results, 
choosing $n_{Syn} = n$ is very safe, and we use it for all algorithms except RAP,
as in the toy data experiment (Section~\ref{sec:toy-data-experiment}). 
For NAPSU-MQ and PGM-repeat, we choose the number of generated synthetic datasets
with a preliminary experiment, presented in
Figures~\ref{fig:max-ent-m-comparison} and \ref{fig:pgm-repeats-m-comparison}
in the Supplement. For NAPSU-MQ, the theory of \textcite{reiter2002satisfying}, 
suggests that a larger $m$ is better, but has diminishing returns. Our results in 
Figure~\ref{fig:max-ent-m-comparison} validate this, as all values of $m$
produce similar results. We describe the other hyperparameters 
in detail in Supplemental Section~\ref{app:hyperparams}.

The Laplace approximation for NAPSU-MQ does not work well for this setting
because many of the queries have small values, so we use 
NUTS~\autocite{hoffmanNoUTurnSamplerAdaptively2014} for posterior inference.
To speed up NUTS, we normalise the posterior before running the inference 
using the mean and covariance of the Laplace approximation 

Results from 20 repeats of the experiment are shown in 
Figure~\ref{fig:adult-calibration-widths}.
PGM, RAP and PEP produce overconfident confidence intervals that do not meet 
the given confidence levels. 
With the repeats, PGM becomes overly conservative, and produces confidence 
intervals that are too wide. NAPSU-MQ is the only algorithm that produces 
properly calibrated intervals, although repeated PGM 
is able to produce narrower intervals than NAPSU-MQ with $\epsilon = 0.1$.
With the more realistic $\epsilon = 1$, the confidence intervals 
from NAPSU-MQ are not much wider than non-DP intervals, while PGM-repeat 
produces much wider intervals. 
Figure~\ref{fig:adult_marginal_accuracy}
shows that NAPSU-MQ reproduces 1- and 2-way marginals nearly as accurately as 
PGM.

The results in Figure~\ref{fig:adult-calibration-widths} were obtained using 
the small regularisation term.
Figure~\ref{fig:adult-calibration-widths-nonregularised} shows 
the results with the trick of dropping large variances, which are very close to 
Figure~\ref{fig:adult-calibration-widths}.

Noise-awareness, especially 
with the increased accuracy from NUTS, comes with a steep computational cost,
as PGM ran in 15s, while the Laplace approximation took several minutes, 
and NUTS took up to ten hours. All of the algorithms were run on 4 CPU cores of 
a cluster. The complete set of runtimes for all algorithms and values of 
$\epsilon$ are shown in Table~\ref{tab:runtime} of the Supplement. 

\section{DISCUSSION}\label{sec:limitation-conclusion}

While our general pipeline NA+MI is applicable to all 
kinds of datasets in principle, the data generation algorithm NAPSU-MQ is 
currently only applicable to discrete tabular data due to the reliance on 
perturbing query values,
and only supports sparse marginal queries perturbed with the Gaussian 
mechanism as input due to the techniques we use to make the algorithm practical. 

We aim to generalise NAPSU-MQ to more general query classes, such as linear queries,
in the future. Getting rid of the dependency on queries completely is likely to 
be much more challenging, as it will require developing noise-aware Bayesian 
inference without perturbing sufficient statistics. This may be possible by
adding noise-awareness to other types of DP Bayesian inference methods, like 
DP variational inference~\autocite{jalkoDifferentiallyPrivateVariational2017} or 
DP MCMC~\autocite{
  heikkilaDifferentiallyPrivateMarkov2019,
  yildirimExactMCMCDifferentially2019,
}.

Only handling discrete data is not a major limitation, as the combination of 
discretisation~\autocite{zhangPrivTreeDifferentiallyPrivate2016} and synthetic 
data generation with marginal-based algorithms like 
PGM~\autocite{mckennaGraphicalmodelBasedEstimation2019}
have been shown to perform very well on tabular synthetic data generation 
tasks~\autocite{taoBenchmarkingDifferentiallyPrivate2021}.

The Gaussian mechanism adds noise uniformly to all queries, making small 
queries relatively more noisy.
This may reduce the accuracy of 
query-based algorithms like NAPSU-MQ and the others we examined for 
groups with rare combinations of data values, such as minorities. 

Although we left query selection outside the scope of this paper, selecting the 
right queries to support downstream analysis is important, as NA+MI 
cannot guarantee confidence interval coverage if the selected queries do not 
contain enough information for the downstream task. We plan to study whether 
existing methods giving confidence bounds on query 
accuracy~\autocite{mckennaAIMAdaptiveIterative2022, nixonLatentClassModeling2022} 
can be adapted to give confidence intervals for arbitrary downstream analyses.

The runtime is a significant limitation
in the current implementation of NAPSU-MQ when using NUTS.
As NAPSU-MQ is compatible with any non-DP posterior sampling method, 
recent~\autocite{hoffmanAdaptiveMCMCSchemeSetting2021} and future advances 
in MCMC and other sampling techniques are likely able to cut down on the 
runtime.

\paragraph{Conclusion} The analysis of DP synthetic data has not received much 
attention in existing research. Our work patches a major hole in the 
current generation and analysis methods by developing the NA+MI pipeline that allows 
computing accurate confidence intervals and $p$-values from DP synthetic data.
We develop the NAPSU-MQ algorithm in order to implement NA+MI on 
nontrivial discrete datasets. NA+MI only depends on noise-aware 
posterior inference, not NAPSU-MQ specifically, and can thus be extended 
to other settings in the future.
With the noise-aware inference algorithm, NA+MI allows conducting valid 
statistical analyses that include uncertainty estimates 
with DP synthetic data, potentially unlocking existing privacy-sensitive datasets 
for widespread analysis.

\subsection*{Acknowledgments}
This work was supported by the Academy of Finland 
(Flagship programme: Finnish Center for Artificial Intelligence, 
FCAI; and grants 325572, 325573), the Strategic Research Council 
at the Academy of Finland (Grant 336032), the UKRI Turing 
AI World-Leading Researcher Fellowship, EP/W002973/1,
as well as the European Union (Project
101070617). Views and opinions expressed are however
those of the author(s) only and do not necessarily reflect
those of the European Union or the European Commission. Neither the European 
Union nor the granting authority can be held responsible for them.
The authors wish to thank the Finnish Computing Competence 
Infrastructure (FCCI) for supporting this project with 
computational and data storage resources.

\defbibheading{bibliography}[References]{\subsubsection*{#1}}
\printbibliography

@article{abbeelLearningFactorGraphs2006,
  title = {Learning Factor Graphs in Polynomial Time and Sample Complexity},
  author = {Abbeel, Pieter and Koller, Daphne and Ng, Andrew Y.},
  year = {2006},
  journal = {Journal of Machine Learning Research},
  volume = {7},
  pages = {1743--1788},
  bibsource = {dblp computer science bibliography, https://dblp.org}
}

@inproceedings{aydoreDifferentiallyPrivateQuery2021,
  title = {Differentially {{Private Query Release Through Adaptive Projection}}},
  booktitle = {Proceedings of the 38th {{International Conference}} on {{Machine Learning}}},
  author = {Aydore, Sergul and Brown, William and Kearns, Michael and Kenthapadi, Krishnaram and Melis, Luca and Roth, Aaron and Siva, Ankit A},
  editor = {Meila, Marina and Zhang, Tong},
  year = {2021},
  month = jul,
  series = {Proceedings of {{Machine Learning Research}}},
  volume = {139},
  pages = {457--467},
  publisher = {{PMLR}}
}

@inproceedings{balleImprovingGaussianMechanism2018,
  title = {Improving the {{Gaussian Mechanism}} for {{Differential Privacy}}: {{Analytical Calibration}} and {{Optimal Denoising}}},
  booktitle = {Proceedings of the 35th {{International Conference}} on {{Machine Learning}}},
  author = {Balle, Borja and Wang, Yu-Xiang},
  editor = {Dy, Jennifer and Krause, Andreas},
  year = {2018},
  month = jul,
  series = {Proceedings of {{Machine Learning Research}}},
  volume = {80},
  pages = {394--403},
  publisher = {{PMLR}}
}

@inproceedings{bernsteinDifferentiallyPrivateBayesian2018,
  title = {Differentially {{Private Bayesian Inference}} for {{Exponential Families}}},
  booktitle = {Advances in {{Neural Information Processing Systems}}},
  author = {Bernstein, Garrett and Sheldon, Daniel},
  editor = {Bengio, Samy and Wallach, Hanna M. and Larochelle, Hugo and Grauman, Kristen and {Cesa-Bianchi}, Nicol{\`o} and Garnett, Roman},
  year = {2018},
  volume = {31},
  pages = {2924--2934}
}

@inproceedings{bernsteinDifferentiallyPrivateBayesian2019,
  title = {Differentially {{Private Bayesian Linear Regression}}},
  booktitle = {Advances in {{Neural Information Processing Systems}}},
  author = {Bernstein, Garrett and Sheldon, Daniel},
  editor = {Wallach, Hanna M. and Larochelle, Hugo and Beygelzimer, Alina and {d'Alch{\'e}-Buc}, Florence and Fox, Emily B. and Garnett, Roman},
  year = {2019},
  volume = {32},
  pages = {523--533}
}

@inproceedings{bernsteinDifferentiallyPrivateLearning2017,
  title = {Differentially {{Private Learning}} of {{Undirected Graphical Models Using Collective Graphical Models}}},
  booktitle = {Proceedings of the 34th {{International Conference}} on {{Machine Learning}}},
  author = {Bernstein, Garrett and McKenna, Ryan and Sun, Tao and Sheldon, Daniel and Hay, Michael and Miklau, Gerome},
  editor = {Precup, Doina and Teh, Yee Whye},
  year = {2017},
  series = {Proceedings of {{Machine Learning Research}}},
  volume = {70},
  pages = {478--487},
  publisher = {{PMLR}}
}

@article{cai2021data,
  title = {Data Synthesis via Differentially Private Markov Random Fields},
  author = {Cai, Kuntai and Lei, Xiaoyu and Wei, Jianxin and Xiao, Xiaokui},
  year = {2021},
  journal = {Proceedings of the VLDB Endowment},
  volume = {14},
  number = {11},
  pages = {2190--2202},
  publisher = {{VLDB Endowment}}
}

@article{charestHowCanWe2010,
  title = {How {{Can We Analyze Differentially-Private Synthetic Datasets}}?},
  author = {Charest, Anne-Sophie},
  year = {2010},
  journal = {Journal of Privacy and Confidentiality},
  volume = {2},
  number = {2},
  doi = {10.29012/jpc.v2i2.589}
}

@inproceedings{chenDifferentiallyPrivateHighDimensional2015,
  title = {Differentially {{Private High-Dimensional Data Publication}} via {{Sampling-Based Inference}}},
  booktitle = {Proceedings of the 21th {{ACM SIGKDD International Conference}} on {{Knowledge Discovery}} and {{Data Mining}}},
  author = {Chen, Rui and Xiao, Qian and Zhang, Yu and Xu, Jianliang},
  editor = {Cao, Longbing and Zhang, Chengqi and Joachims, Thorsten and Webb, Geoffrey I. and Margineantu, Dragos D. and Williams, Graham},
  year = {2015},
  pages = {129--138},
  publisher = {{ACM}},
  doi = {10.1145/2783258.2783379}
}

@inproceedings{chenGSWGANGradientsanitizedApproach2020,
  title = {{{GS-WGAN}}: {{A}} Gradient-Sanitized Approach for Learning Differentially Private Generators},
  booktitle = {Advances in {{Neural Information Processing Systems}}},
  author = {Chen, Dingfan and Orekondy, Tribhuvanesh and Fritz, Mario},
  editor = {Larochelle, H. and Ranzato, M. and Hadsell, R. and Balcan, M.F. and Lin, H.},
  year = {2020},
  volume = {33},
  pages = {12673--12684}
}

@misc{covingtonUnbiasedStatisticalEstimation2021,
  type = {{{arXiv}} Preprint},
  title = {Unbiased {{Statistical Estimation}} and {{Valid Confidence Intervals Under Differential Privacy}}},
  author = {Covington, Christian and He, Xi and Honaker, James and Kamath, Gautam},
  year = {2021},
  month = oct,
  number = {arXiv:2110.14465},
  eprint = {2110.14465},
  eprinttype = {arxiv},
  publisher = {{arXiv}},
  archiveprefix = {arXiv}
}

@article{dworkAlgorithmicFoundationsDifferential2014,
  title = {The {{Algorithmic Foundations}} of {{Differential Privacy}}},
  author = {Dwork, Cynthia and Roth, Aaron},
  year = {2014},
  journal = {Foundations and Trends in Theoretical Computer Science},
  volume = {9},
  number = {3-4},
  pages = {211--407},
  doi = {10.1561/0400000042}
}

@inproceedings{dworkCalibratingNoiseSensitivity2006,
  title = {Calibrating {{Noise}} to {{Sensitivity}} in {{Private Data Analysis}}},
  booktitle = {Third {{Theory}} of {{Cryptography Conference}}},
  author = {Dwork, Cynthia and McSherry, Frank and Nissim, Kobbi and Smith, Adam D.},
  editor = {Halevi, Shai and Rabin, Tal},
  year = {2006},
  series = {Lecture {{Notes}} in {{Computer Science}}},
  volume = {3876},
  pages = {265--284},
  publisher = {{Springer}},
  doi = {10.1007/11681878_14}
}

@inproceedings{dworkDifferentialPrivacySurvey2008,
  title = {Differential Privacy: {{A}} Survey of Results},
  booktitle = {International Conference on Theory and Applications of Models of Computation},
  author = {Dwork, Cynthia},
  year = {2008},
  pages = {1--19},
  publisher = {{Springer}}
}

@inproceedings{dworkOurDataOurselves2006,
  title = {Our {{Data}}, {{Ourselves}}: {{Privacy Via Distributed Noise Generation}}},
  booktitle = {Advances in {{Cryptology}} - {{EUROCRYPT}}},
  author = {Dwork, Cynthia and Kenthapadi, Krishnaram and McSherry, Frank and Mironov, Ilya and Naor, Moni},
  editor = {Vaudenay, Serge},
  year = {2006},
  series = {Lecture {{Notes}} in {{Computer Science}}},
  volume = {4004},
  pages = {486--503},
  publisher = {{Springer}},
  doi = {10.1007/11761679_29}
}

@article{evansStatisticallyValidInferences2022,
  title = {Statistically {{Valid Inferences}} from {{Differentially Private Data Releases}}, with {{Application}} to the {{Facebook URLs Dataset}}},
  author = {Evans, Georgina and King, Gary},
  year = {2022},
  month = apr,
  journal = {Political Analysis},
  pages = {1--21},
  issn = {1047-1987, 1476-4989},
  doi = {10.1017/pan.2022.1},
  langid = {english}
}

@inproceedings{ferrandoParametricBootstrapDifferentially2022,
  title = {Parametric Bootstrap for Differentially Private Confidence Intervals},
  booktitle = {Proceedings of {{The}} 25th {{International Conference}} on {{Artificial Intelligence}} and {{Statistics}}},
  author = {Ferrando, Cecilia and Wang, Shufan and Sheldon, Daniel},
  editor = {{Camps-Valls}, Gustau and Ruiz, Francisco J. R. and Valera, Isabel},
  year = {2022},
  month = mar,
  series = {Proceedings of Machine Learning Research},
  volume = {151},
  pages = {1598--1618},
  publisher = {{PMLR}}
}

@book{gelmanBayesianDataAnalysis2014,
  title = {Bayesian Data Analysis},
  author = {Gelman, Andrew and Carlin, John B and Stern, Hal S and Dunson, David B and Vehtari, Aki and Rubin, Donald B},
  year = {2014},
  journal = {Bayesian data analysis},
  series = {Chapman \& {{Hall}}/{{CRC}} Texts in Statistical Science Series},
  edition = {Third edition},
  publisher = {{CRC Press}},
  address = {{Boca Raton}},
  isbn = {978-1-4398-9820-8},
  langid = {english}
}

@misc{ghalebikesabiBiasMitigatedLearning2021,
  type = {{{arXiv}} Preprint},
  title = {Bias {{Mitigated Learning}} from {{Differentially Private Synthetic Data}}: {{A Cautionary Tale}}},
  author = {Ghalebikesabi, Sahra and Wilde, Harrison and Jewson, Jack and Doucet, Arnaud and Vollmer, Sebastian J. and Holmes, Chris},
  year = {2021},
  number = {arXiv: 2108.10934},
  eprint = {2108.10934},
  eprinttype = {arxiv},
  publisher = {{arXiv}},
  archiveprefix = {arXiv}
}

@article{gongExactInferenceApproximate2022,
  title = {Exact {{Inference}} with {{Approximate Computation}} for {{Differentially Private Data}} via {{Perturbations}}},
  author = {Gong, Ruobin},
  year = {2022},
  month = nov,
  journal = {Journal of Privacy and Confidentiality},
  volume = {12},
  number = {2},
  issn = {2575-8527},
  doi = {10.29012/jpc.797},
  copyright = {Copyright (c) 2022 Ruobin Gong},
  langid = {english}
}

@inproceedings{harderDPMERFDifferentiallyPrivate2021,
  title = {{{DP-MERF}}: {{Differentially Private Mean Embeddings}} with {{RandomFeatures}} for {{Practical Privacy-preserving Data Generation}}},
  booktitle = {Proceedings of {{The}} 24th {{International Conference}} on {{Artificial Intelligence}} and {{Statistics}}},
  author = {Harder, Frederik and Adamczewski, Kamil and Park, Mijung},
  editor = {Banerjee, Arindam and Fukumizu, Kenji},
  year = {2021},
  month = apr,
  series = {Proceedings of {{Machine Learning Research}}},
  volume = {130},
  pages = {1819--1827},
  publisher = {{PMLR}}
}

@inproceedings{hardtSimplePracticalAlgorithm2012,
  title = {A {{Simple}} and {{Practical Algorithm}} for {{Differentially Private Data Release}}},
  booktitle = {Advances in {{Neural Information Processing Systems}}},
  author = {Hardt, Moritz and Ligett, Katrina and McSherry, Frank},
  editor = {Bartlett, Peter L. and Pereira, Fernando C. N. and Burges, Christopher J. C. and Bottou, L{\'e}on and Weinberger, Kilian Q.},
  year = {2012},
  volume = {25},
  pages = {2348--2356}
}

@inproceedings{heikkilaDifferentiallyPrivateMarkov2019,
  title = {Differentially {{Private Markov Chain Monte Carlo}}},
  booktitle = {Advances in {{Neural Information Processing Systems}}},
  author = {Heikkil{\"a}, Mikko A. and J{\"a}lk{\"o}, Joonas and Dikmen, Onur and Honkela, Antti},
  year = {2019},
  volume = {32},
  pages = {4115--4125}
}

@inproceedings{hoffmanAdaptiveMCMCSchemeSetting2021,
  title = {An {{Adaptive-MCMC Scheme}} for {{Setting Trajectory Lengths}} in {{Hamiltonian Monte Carlo}}},
  booktitle = {The 24th {{International Conference}} on {{Artificial Intelligence}} and {{Statistics}}},
  author = {Hoffman, Matthew D. and Radul, Alexey and Sountsov, Pavel},
  editor = {Banerjee, Arindam and Fukumizu, Kenji},
  year = {2021},
  series = {Proceedings of {{Machine Learning Research}}},
  volume = {130},
  pages = {3907--3915},
  publisher = {{PMLR}}
}

@article{hoffmanNoUTurnSamplerAdaptively2014,
  title = {The {{No-U-Turn}} Sampler: Adaptively Setting Path Lengths in {{Hamiltonian Monte Carlo}}.},
  author = {Hoffman, Matthew D and Gelman, Andrew},
  year = {2014},
  journal = {Journal of Machine Learning Research},
  volume = {15},
  number = {1},
  pages = {1593--1623}
}

@inproceedings{jalkoDifferentiallyPrivateVariational2017,
  title = {Differentially {{Private Variational Inference}} for {{Non-conjugate Models}}},
  booktitle = {Proceedings of the {{Thirty-Third Conference}} on {{Uncertainty}} in {{Artificial Intelligence}}},
  author = {J{\"a}lk{\"o}, Joonas and Honkela, Antti and Dikmen, Onur},
  editor = {Elidan, Gal and Kersting, Kristian and Ihler, Alexander T.},
  year = {2017},
  publisher = {{AUAI Press}}
}

@article{jalkoPrivacypreservingDataSharing2021,
  title = {Privacy-Preserving Data Sharing via Probabilistic Modeling},
  author = {J{\"a}lk{\"o}, Joonas and Lagerspetz, Eemil and Haukka, Jari and Tarkoma, Sasu and Honkela, Antti and Kaski, Samuel},
  year = {2021},
  month = jul,
  journal = {Patterns},
  volume = {2},
  number = {7},
  pages = {100271},
  issn = {26663899},
  doi = {10.1016/j.patter.2021.100271},
  langid = {english}
}

@article{jaynesInformationTheoryStatistical1957,
  title = {Information {{Theory}} and {{Statistical Mechanics}}},
  author = {Jaynes, E. T.},
  year = {1957},
  month = may,
  journal = {Physical Review},
  volume = {106},
  number = {4},
  pages = {620--630},
  issn = {0031-899X},
  doi = {10.1103/PhysRev.106.620},
  langid = {english}
}

@inproceedings{kairouzCompositionTheoremDifferential2015,
  title = {The {{Composition Theorem}} for {{Differential Privacy}}},
  booktitle = {Proceedings of the 32nd {{International Conference}} on {{Machine Learning}}},
  author = {Kairouz, Peter and Oh, Sewoong and Viswanath, Pramod},
  editor = {Bach, Francis and Blei, David},
  year = {2015},
  month = jul,
  series = {Proceedings of {{Machine Learning Research}}},
  volume = {37},
  pages = {1376--1385},
  publisher = {{PMLR}}
}

@misc{kohaviAdult1996,
  title = {Adult},
  author = {Kohavi, Ronny and Becker, Barry},
  year = {1996},
  copyright = {CC BY 4.0},
  howpublished = {UCI Machine Learning Repository}
}

@book{koller2009probabilistic,
  title = {Probabilistic Graphical Models: Principles and Techniques},
  author = {Koller, Daphne and Friedman, Nir},
  year = {2009},
  publisher = {{MIT press}}
}

@inproceedings{kulkarniDifferentiallyPrivateBayesian2021,
  title = {Differentially {{Private Bayesian Inference}} for {{Generalized Linear Models}}},
  booktitle = {Proceedings of the 38th {{International Conference}} on {{Machine Learning}}},
  author = {Kulkarni, Tejas and J{\"a}lk{\"o}, Joonas and Koskela, Antti and Kaski, Samuel and Honkela, Antti},
  editor = {Meila, Marina and Zhang, Tong},
  year = {2021},
  month = jul,
  series = {Proceedings of {{Machine Learning Research}}},
  volume = {139},
  pages = {5838--5849},
  publisher = {{PMLR}}
}

@inproceedings{liewPEARLDataSynthesis2022,
  title = {{{PEARL}}: {{Data}} Synthesis via Private Embeddings and Adversarial Reconstruction Learning},
  booktitle = {International {{Conference}} on {{Learning Representations}}},
  author = {Liew, Seng Pei and Takahashi, Tsubasa and Ueno, Michihiko},
  year = {2022}
}

@inproceedings{liuIterativeMethodsPrivate2021,
  title = {Iterative {{Methods}} for {{Private Synthetic Data}}: {{Unifying Framework}} and {{New Methods}}},
  shorttitle = {Iterative {{Methods}} for {{Private Synthetic Data}}},
  booktitle = {Advances in {{Neural Information Processing Systems}}},
  author = {Liu, Terrance and Vietri, Giuseppe and Wu, Steven Z.},
  year = {2021},
  volume = {34},
  pages = {690--702}
}

@article{liuModelbasedDifferentiallyPrivate2022,
  title = {Model-Based {{Differentially Private Data Synthesis}} and {{Statistical Inference}} in {{Multiple Synthetic Datasets}}},
  author = {Liu, Fang},
  year = {2022},
  journal = {Trans. Data Priv.},
  volume = {15},
  number = {3},
  pages = {141--175}
}

@inproceedings{longGPATEScalableDifferentially2021,
  title = {G-{{PATE}}: {{Scalable Differentially Private Data Generator}} via {{Private Aggregation}} of {{Teacher Discriminators}}},
  shorttitle = {G-{{PATE}}},
  booktitle = {Advances in {{Neural Information Processing Systems}}},
  author = {Long, Yunhui and Wang, Boxin and Yang, Zhuolin and Kailkhura, Bhavya and Zhang, Aston and Gunter, Carl and Li, Bo},
  year = {2021},
  volume = {34},
  pages = {2965--2977}
}

@article{marshallCombiningEstimatesInterest2009,
  title = {Combining Estimates of Interest in Prognostic Modelling Studies after Multiple Imputation: Current Practice and Guidelines},
  shorttitle = {Combining Estimates of Interest in Prognostic Modelling Studies after Multiple Imputation},
  author = {Marshall, Andrea and Altman, Douglas G and Holder, Roger L and Royston, Patrick},
  year = {2009},
  month = dec,
  journal = {BMC Medical Research Methodology},
  volume = {9},
  number = {1},
  pages = {57},
  issn = {1471-2288},
  doi = {10.1186/1471-2288-9-57},
  langid = {english}
}

@article{mckenna2018optimizing,
  title = {Optimizing Error of High-Dimensional Statistical Queries under Differential Privacy},
  author = {McKenna, Ryan and Miklau, Gerome and Hay, Michael and Machanavajjhala, Ashwin},
  year = {2018},
  journal = {Proceedings of the VLDB Endowment},
  volume = {11},
  number = {10}
}

@article{mckennaAIMAdaptiveIterative2022,
  title = {{{AIM}}: An Adaptive and Iterative Mechanism for Differentially Private Synthetic Data},
  shorttitle = {{{AIM}}},
  author = {McKenna, Ryan and Mullins, Brett and Sheldon, Daniel and Miklau, Gerome},
  year = {2022},
  month = sep,
  journal = {Proceedings of the VLDB Endowment},
  volume = {15},
  number = {11},
  pages = {2599--2612},
  issn = {2150-8097},
  doi = {10.14778/3551793.3551817}
}

@inproceedings{mckennaGraphicalmodelBasedEstimation2019,
  title = {Graphical-Model Based Estimation and Inference for Differential Privacy},
  booktitle = {Proceedings of the 36th {{International Conference}} on {{Machine Learning}}},
  author = {McKenna, Ryan and Sheldon, Daniel and Miklau, Gerome},
  editor = {Chaudhuri, Kamalika and Salakhutdinov, Ruslan},
  year = {2019},
  series = {Proceedings of {{Machine Learning Research}}},
  volume = {97},
  pages = {4435--4444},
  publisher = {{PMLR}}
}

@article{mckennaWinningNISTContest2021a,
  title = {Winning the {{NIST Contest}}: {{A}} Scalable and General Approach to Differentially Private Synthetic Data},
  shorttitle = {Winning the {{NIST Contest}}},
  author = {McKenna, Ryan and Miklau, Gerome and Sheldon, Daniel},
  year = {2021},
  month = dec,
  journal = {Journal of Privacy and Confidentiality},
  volume = {11},
  number = {3},
  issn = {2575-8527},
  doi = {10.29012/jpc.778}
}

@misc{meekUSCensusData2001,
  title = {{{US Census Data}} (1990)},
  author = {Meek, Chris and Thiesson, Bo and Heckerman, David},
  year = {2001},
  copyright = {CC BY 4.0},
  howpublished = {UCI Machine Learning Repository}
}

@inproceedings{neunhoefferPrivatePostGANBoosting2021,
  title = {Private {{Post-GAN Boosting}}},
  booktitle = {International {{Conference}} on {{Learning Representations}}},
  author = {Neunhoeffer, Marcel and Wu, Steven and Dwork, Cynthia},
  year = {2021}
}

@article{nixonLatentClassModeling2022,
  title = {A {{Latent Class Modeling Approach}} for {{Generating Synthetic Data}} and {{Making Posterior Inferences}} from {{Differentially Private Counts}}},
  author = {Nixon, Michelle and Barrientos, Andres and Reiter, Jerome and Slavkovic, Aleksandra},
  year = {2022},
  month = jul,
  journal = {Journal of Privacy and Confidentiality},
  volume = {12},
  number = {1},
  issn = {2575-8527},
  doi = {10.29012/jpc.768},
  copyright = {Copyright (c) 2022 Michelle Nixon, Andres Barrientos, Jerome Reiter, Aleksandra Slavkovic},
  langid = {english}
}

@article{raabPracticalDataSynthesis2018,
  title = {Practical {{Data Synthesis}} for {{Large Samples}}},
  author = {Raab, Gillian M and Nowok, Beata and Dibben, Chris},
  year = {2018},
  month = feb,
  journal = {Journal of Privacy and Confidentiality},
  volume = {7},
  number = {3},
  pages = {67--97},
  issn = {2575-8527},
  doi = {10.29012/jpc.v7i3.407}
}

@article{raghunathanMultipleImputationStatistical2003,
  title = {Multiple Imputation for Statistical Disclosure Limitation},
  author = {Raghunathan, Trivellore E and Reiter, Jerome P and Rubin, Donald B},
  year = {2003},
  journal = {Journal of Official Statistics},
  volume = {19},
  number = {1},
  pages = {1},
  publisher = {{Statistics Sweden (SCB)}}
}

@article{reiter2002satisfying,
  title = {Satisfying Disclosure Restrictions with Synthetic Data Sets},
  author = {Reiter, Jerome P},
  year = {2002},
  journal = {Journal of Official Statistics},
  volume = {18},
  number = {4},
  pages = {531},
  publisher = {{Statistics Sweden (SCB)}}
}

@article{reiter2005significance,
  title = {Significance Tests for Multi-Component Estimands from Multiply Imputed, Synthetic Microdata},
  author = {Reiter, Jerome P},
  year = {2005},
  journal = {Journal of Statistical Planning and Inference},
  volume = {131},
  number = {2},
  pages = {365--377},
  publisher = {{Elsevier}}
}

@article{rubin1993statistical,
  title = {Discussion: {{Statistical}} Disclosure Limitation},
  author = {Rubin, Donald B.},
  year = {1993},
  journal = {Journal of Official Statistics},
  volume = {9},
  number = {2},
  pages = {461--468}
}

@incollection{rubinMultipleImputationNonresponse1987,
  title = {Multiple Imputation for Nonresponse in Surveys},
  booktitle = {Multiple Imputation for Nonresponse in Surveys},
  author = {Rubin, Donald B.},
  year = {1987},
  publisher = {{John Wiley \textbackslash\& Sons}},
  address = {{New York}},
  isbn = {0-471-08705-X},
  langid = {english}
}

@article{si2011comparison,
  title = {A Comparison of Posterior Simulation and Inference by Combining Rules for Multiple Imputation},
  author = {Si, Yajuan and Reiter, Jerome P},
  year = {2011},
  journal = {Journal of Statistical Theory and Practice},
  volume = {5},
  number = {2},
  pages = {335--347},
  publisher = {{Taylor \& Francis}}
}

@misc{taoBenchmarkingDifferentiallyPrivate2021,
  type = {{{arXiv}} Preprint},
  title = {Benchmarking {{Differentially Private Synthetic Data Generation Algorithms}}},
  author = {Tao, Yuchao and McKenna, Ryan and Hay, Michael and Machanavajjhala, Ashwin and Miklau, Gerome},
  year = {2021},
  number = {arXiv: 2112.09238},
  eprint = {2112.09238},
  eprinttype = {arxiv},
  publisher = {{arXiv}},
  archiveprefix = {arXiv}
}

@inproceedings{vietriNewOracleEfficientAlgorithms2020,
  title = {New {{Oracle-Efficient Algorithms}} for {{Private Synthetic Data Release}}},
  booktitle = {Proceedings of the 37th {{International Conference}} on {{Machine Learning}}},
  author = {Vietri, Giuseppe and Tian, Grace and Bun, Mark and Steinke, Thomas and Wu, Zhiwei Steven},
  year = {2020},
  series = {Proceedings of {{Machine Learning Research}}},
  volume = {119},
  pages = {9765--9774},
  publisher = {{PMLR}}
}

@article{wainwrightGraphicalModelsExponential2008,
  title = {Graphical {{Models}}, {{Exponential Families}}, and {{Variational Inference}}},
  author = {Wainwright, Martin J. and Jordan, Michael I.},
  year = {2008},
  journal = {Foundations and Trends in Machine Learning},
  volume = {1},
  number = {1-2},
  pages = {1--305},
  doi = {10.1561/2200000001}
}

@inproceedings{wildeFoundationsBayesianLearning2021,
  title = {Foundations of {{Bayesian Learning}} from {{Synthetic Data}}},
  booktitle = {The 24th {{International Conference}} on {{Artificial Intelligence}} and {{Statistics}}},
  author = {Wilde, Harrison and Jewson, Jack and Vollmer, Sebastian J. and Holmes, Chris},
  editor = {Banerjee, Arindam and Fukumizu, Kenji},
  year = {2021},
  series = {Proceedings of {{Machine Learning Research}}},
  volume = {130},
  pages = {541--549},
  publisher = {{PMLR}}
}

@misc{xieDifferentiallyPrivateGenerative2018,
  type = {{{arXiv}} Preprint},
  title = {Differentially {{Private Generative Adversarial Network}}},
  author = {Xie, Liyang and Lin, Kaixiang and Wang, Shu and Wang, Fei and Zhou, Jiayu},
  year = {2018},
  month = feb,
  number = {arXiv:1802.06739},
  eprint = {1802.06739},
  eprinttype = {arxiv},
  publisher = {{arXiv}},
  archiveprefix = {arXiv}
}

@article{yildirimExactMCMCDifferentially2019,
  title = {Exact {{MCMC}} with Differentially Private Moves - {{Revisiting}} the Penalty Algorithm in a Data Privacy Framework},
  author = {Yildirim, Sinan and Ermis, Beyza},
  year = {2019},
  journal = {Statistics and Computing},
  volume = {29},
  number = {5},
  pages = {947--963},
  doi = {10.1007/s11222-018-9847-x}
}

@inproceedings{yoon2018pategan,
  title = {{{PATE-GAN}}: {{Generating}} Synthetic Data with Differential Privacy Guarantees},
  booktitle = {International {{Conference}} on {{Learning Representations}}},
  author = {Yoon, Jinsung and Jordon, James and {van der Schaar}, Mihaela},
  year = {2019}
}

@article{zhangPrivBayesPrivateData2017,
  title = {{{PrivBayes}}: {{Private Data Release}} via {{Bayesian Networks}}},
  author = {Zhang, Jun and Cormode, Graham and Procopiuc, Cecilia M. and Srivastava, Divesh and Xiao, Xiaokui},
  year = {2017},
  journal = {ACM Transactions on Database Systems},
  volume = {42},
  number = {4},
  pages = {25:1--25:41},
  doi = {10.1145/3134428}
}

@inproceedings{zhangPrivTreeDifferentiallyPrivate2016,
  title = {{{PrivTree}}: {{A Differentially Private Algorithm}} for {{Hierarchical Decompositions}}},
  booktitle = {Proceedings of the 2016 {{International Conference}} on {{Management}} of {{Data}}, {{SIGMOD Conference}}},
  author = {Zhang, Jun and Xiao, Xiaokui and Xie, Xing},
  editor = {{\"O}zcan, Fatma and Koutrika, Georgia and Madden, Sam},
  year = {2016},
  pages = {155--170},
  publisher = {{ACM}},
  doi = {10.1145/2882903.2882928}
}

@phdthesis{zhengDifferentialPrivacyBayesian2015,
  type = {Bachelor's Thesis},
  title = {The Differential Privacy of {{Bayesian}} Inference},
  author = {Zheng, Shijie},
  year = {2015},
  school = {Harvard College}
}

@article{zhouNoteBayesianInference2010,
  title = {A {{Note}} on {{Bayesian Inference After Multiple Imputation}}},
  author = {Zhou, Xiang and Reiter, Jerome P.},
  year = {2010},
  month = may,
  journal = {The American Statistician},
  volume = {64},
  number = {2},
  pages = {159--163},
  issn = {0003-1305, 1537-2731},
  doi = {10.1198/tast.2010.09109},
  langid = {english}
}

\appendix
\renewcommand\thefigure{S\arabic{figure}}
\renewcommand\thetable{S\arabic{table}}
\setcounter{figure}{0}

\onecolumn 

\section{PROOF OF THEOREM~\ref{thm:marginal-sensitivity}}\label{app:proofs}
\theoremmarginalsensitivity*
\begin{proof}
  Let $a_1, \dotsc, a_{n_s}$ be the full sets of marginal queries that form 
  $a$. Because all of the queries of $a_i$ have the same set of variables,
  the vector $a_i(x)$ has a single component of value 1, and the other components 
  are 0 for any $x\in \calx$. Then, for any neighbouring $X, X'\in \calx^n$, 
  $||a_i(X) - a_i(X')||_2^2 \leq 2$.
  Then 
  \begin{align}
    \Delta_2 a &= \sup_{X\sim X'}||a(X) - a(X')||_2
    = \sup_{X\sim X'}\sqrt{\sum_{i=1}^{n_s}||a_i(X) - a_i(X')||_2^2}
    \leq \sup_{X\sim X'}\sqrt{\sum_{i=1}^{n_s} 2}
    = \sqrt{2n_s}
  \end{align}
\end{proof}

\section{MULTIPLE IMPUTATION}\label{app:mi}
In order to compute uncertainty estimates for downstream analyses from the 
noise-aware posterior with NA+MI, we use Rubin's rules for synthetic 
data~\autocite{raghunathanMultipleImputationStatistical2003, reiter2002satisfying}. 

After the synthetic datasets $\xsyn_i$ for $1\leq i\leq m$ are released by 
the data holder, the data analyst runs their downstream analysis on 
each $\xsyn_i$. For each synthetic dataset, the analysis produces a 
point estimate $q_i$ and a variance estimate $v_i$ for $q_i$.

The estimates $q_1, \dotsc, q_m$ and 
$v_1, \dotsc, v_m$ are combined as 
follows~\autocite{raghunathanMultipleImputationStatistical2003}:
\begin{align}
  \bar{q} = \frac{1}{m}\sum_{i=1}^m q_i, \quad
  \bar{v} = \frac{1}{m}\sum_{i=1}^m v_i, \quad
  b = \frac{1}{m - 1}\sum_{i=1}^m (q_i - \bar{q})^2.
  \label{eq:rubin-combining}
\end{align}

We use $\bar{q}$ as the combined point estimate, and set
\begin{equation}
  T = \left(1 + \frac{1}{m}\right)b - \bar{v},
  \quad \quad
  T^* =
  \begin{cases}
    T \quad \text{if } T \geq 0 \\
    \frac{n_{Syn}}{n}\bar{v} \quad \text{otherwise}.
  \end{cases}
  \label{eq:rubin-variance}
\end{equation}
$T$ is an estimate of the combined variance. $T$ can be negative, which is corrected 
using $T^*$ instead~\autocite{reiter2002satisfying}.

We compute confidence intervals and 
hypothesis tests using the $t$-distribution with mean $\bar{q}$, 
variance $T^*$, and degrees of freedom 
\begin{equation}
  \nu = (m - 1)(1 - r^{-1})^2,
\end{equation}
where $r = (1 + \frac{1}{m})\frac{b}{\bar{v}}$~\autocite{reiter2002satisfying}.

These combining rules apply when $q$ is a univariate estimate.
\textcite{reiter2005significance} derives appropriate combining rules for 
multivariate estimates, which can be applied with NA+MI.

Rubin's rules make many assumptions on the different distributions that
are involved~\autocite{raghunathanMultipleImputationStatistical2003,si2011comparison},
such as the normality of the distribution of $q_i$ when sampling data from 
the population. These assumptions may not hold for some types of estimates, 
such as probabilities~\autocite{marshallCombiningEstimatesInterest2009} 
or quantile estimates~\autocite{zhouNoteBayesianInference2010}.
Further work~\autocite{si2011comparison} 
tries to reduce these assumptions, especially in the context of 
missing data. Their results for synthetic data generation can be applied 
with our method.

\textcite{si2011comparison} propose to remove some of these assumptions 
by approximating the integral that Rubin's rules are derived from by sampling 
instead of using the analytical approximations in
\eqref{eq:rubin-combining} and \eqref{eq:rubin-variance}. They find that 
their sampling-based approximation can be effective, especially with a 
small number of datasets, but is computationally more expensive.

\subsection{Unbiasedness of Rubin's Rules}

Rubin's rules make several assumptions on the downstream analysis method, and several normal 
approximations when deriving the rules. \textcite{raghunathanMultipleImputationStatistical2003}
derive conditions under which Rubin's rules give an unbiased estimate.

Rubin's rules aim to estimate a quantity $Q$ of the entire population $P$, 
of which $X$ is a sample. Conceptually, the sampling of the synthetic datasets 
is done in two stages: first, synthetic populations $P_i^{Syn}$ for 
$1 \leq i \leq m$ are sampled. Second,
a synthetic dataset $X_i^{Syn}$ is sampled i.i.d from $P_i^{Syn}$. This is 
equivalent to the sampling process for $X_i^{Syn}$ described in 
Section~\ref{sec:na+mi}, and makes stating the assumptions of Rubin's rules 
easier.

While the sampling process of synthetic data has to be i.i.d., this is not 
required for the original data. This means that there are two versions of the 
downstream analysis: the one for i.i.d. data, and the one for the complex 
sampling method of the real data. The latter method is not used at any point, 
so it is assumed to exist for the theory, but does not have to be practically 
implementable.

We take advantage of this handling of complex sampling of real data by 
including the computation of $s$ and adding noise to get $\privs$ in the sampling
scheme, so we are considering $\privs$ to be the original data from the point of 
view of the theory. The fact that $\privs$ is a noisy summary statistic instead
of a dataset is not an issue, as the theory only requires having a theoretical 
method to estimate $Q$ from $\privs$. If the chosen marginal queries 
contain the relevant information for the downstream task, so $s$ is a 
(approximate) sufficient statistic, this theoretical method will exist.

Let $Q_i$ denote the quantity of interest $Q$ computed from the synthetic population 
$P_i^{Syn}$ instead of $P$. Let $V_i$ denote the sampling variance of $q_i$ from the 
synthetic population $P^{Syn}_i$. Note that $q_i$ and it's variance estimate 
$v_i$ are obtained using the downstream analysis method for i.i.d. data.
Let $\hat{Q}_D$ and $\hat{U}_D$ be the point and variance 
estimates of $Q$ derived from $\privs$ when sampling the population $P$, which 
are obtained using the theoretical inference method for complex samples.

Now the assumptions of \textcite{raghunathanMultipleImputationStatistical2003}
are

\begin{assumption}\label{assu:rubin-1}
  For all $1\leq i \leq m$, $q_i$ is unbiased for $Q_i$ and asymptotically normal with 
  respect to sampling from the synthetic population $P_i^{Syn}$, with sampling variance $V_i$.
\end{assumption}

\begin{assumption}\label{assu:rubin-2}
  For all $1 \leq i \leq m$, $v_i$ is unbiased for $V_i$, and the sampling variability in 
  $v_i$ is negligible. That is $v_i \mid P_i^{Syn} \approx V_i$. Additionally, the variation in 
  $V_i$ across the synthetic populations is negligible.
\end{assumption}

Assumptions \ref{assu:rubin-1}-\ref{assu:rubin-2} ensure that the downstream 
analysis method used to estimate $Q$ is accurate, for both point and variance 
estimates, when applied to i.i.d. real data, regardless of the population.

\begin{assumption}\label{assu:rubin-3}
  $\hat{Q}_D \mid P \sim \caln(Q, \hat{U}_D)$
\end{assumption}

Assumption \ref{assu:rubin-3} ensures that the analysis for complex sampling 
is accurate for point and variance estimates when applied to the real 
population.

\begin{assumption}\label{assu:rubin-4}
  $Q_i \mid \privs \sim \caln(\hat{Q}_D, \hat{U}_D)$
\end{assumption}

Assumption \ref{assu:rubin-4} requires that the generation of synthetic 
datasets does not bias the downstream analysis. 

For query-based methods like NAPSU-MQ, Assumptions~\ref{assu:rubin-3} and 
\ref{assu:rubin-4} may not hold when the queries do not contain the relevant 
information for the downstream task.

With Assumptions \ref{assu:rubin-1}-\ref{assu:rubin-4}, 
\textcite{raghunathanMultipleImputationStatistical2003} show that 
$\bar{q}$ is an unbiased estimate of $Q$, and $T$ is an asymptotically unbiased 
variance estimate.

\begin{theorem}[\textcite{raghunathanMultipleImputationStatistical2003}]
  Assumptions \ref{assu:rubin-1}-\ref{assu:rubin-4} imply that
  \begin{enumerate}
    \item 
    $E(\bar{q} \mid P) = Q$,
    \item 
    $E(T\mid P) = \mathrm{Var}(\bar{q}\mid P)$,
    \item Asymptotically $\frac{\bar{q} - Q}{\sqrt{T}} \sim \caln(0, 1)$,
    \item For moderate $m$, 
    $\frac{\bar{q} - Q}{\sqrt{T}} \sim t_\nu(0, 1)$~\autocite{reiter2002satisfying}.
  \end{enumerate}
\end{theorem}

\section{FINDING AND IDENTIFIABLE PARAMETRISATION}\label{sec:canon-queries}
In this section, we describe the process we use to ensure the parametrisation
of the posterior in NAPSU-MQ is identifiable. We ensure 
identifiability by dropping some of the selected queries, chosen using the 
the canonical parametrisation of $\medl$ to ensure 
no information is lost. First, we give some background on Markov networks, 
which is necessary to understand the canonical parametrisation.

\paragraph{Markov Networks}
A Markov network is a representation of a probability distribution that is 
factored according to an undirected graph. Specifically, a Markov network 
distribution $P$ is a product of \emph{factors}. A factor is a function
from a subset of the variables to non-negative real numbers. The subset of 
variables is called the \emph{scope} of the factor. The joint distribution 
is given by 
\begin{equation}
  P(x) = \frac{1}{Z}\prod_{I\subset S}\phi_I(x_I)
\end{equation}
where $S$ is the set of scopes for the factors. The undirected graph is formed 
by representing each variable as a node, and adding edges such that the 
scope of each factor is a clique in the graph.

\paragraph{Canonical Parametrisation}
The canonical parametrisation is given in terms of 
\emph{canonical factors}~\autocite{abbeelLearningFactorGraphs2006}.
The canonical factors depend on an arbitrary assignment of variables 
$x^*$. We simply choose $x^* = (0, \dotsc, 0)$. In the following, 
$x_U$ denotes the selection of components in the set $U$ from the vector $x$,
and $x_{-U}$ denotes the selection of all components except those in $U$.

\begin{definition}
  A canonical factor $\phi^*_D$ with scope $D$ is defined as 
  \[
  \phi^*_D(x) = \exp\left(\sum_{U\subseteq D} (-1)^{|D - U|} \ln P(x_U, x^*_{-U})\right)
  \]
  The sum is over all subsets of $D$, including $D$ itself and the empty set.
  $|D - U|$ is the size of the set difference of $D$ and $U$.
\end{definition}

\begin{theorem}[\textcite{abbeelLearningFactorGraphs2006}(Theorem 3)]
  Let $P$ be a Markov network with factor scopes $S$. Let 
  $S^* = \cup_{D\in S}\mathcal{P}(D) - \emptyset$. Then 
  \[
  P(x) = P(x^*)\prod_{D^*\in S^*} \phi^*_{D^*}(x_{D^*})
  \]
\end{theorem}

There are more canonical factors than original factors, so it might seem 
that there are more parameters in the canonical parametrisation than in the 
original parametrisation. However, many values in the canonical factors 
turn out to be ones. We can select the queries corresponding to non-one 
canonical factor values to obtain a set of queries with the same information 
as the original queries, but without linear 
dependencies~\autocite{koller2009probabilistic}. We call this set of queries 
the \emph{canonical queries}.

Many of the canonical factor scopes are subsets of the original factor 
scopes, so using the canonical queries as is would introduce new 
marginal query sets and potentially increase the sensitivity of the queries. As all 
of the new queries are sums of existing queries, we can replace each new query 
with the old queries that sum to the new query, and use the same $\theta$
value for all of the added queries to preserve identifiability.
After the replacements, the queries are a subset of the original non-canonical 
queries, so their sensitivity is at most the original sensitivity.
If one of the added queries was already included, it does not need to be 
added again, because two instances of a single query can be collapsed into 
a single instance with it's own parameter value. Because of this, we did not 
need to fix the $\theta$ values of any queries to the same value in 
the settings we studied.

\section{NAPSU-MQ VS. PGM}\label{app:napsu-mq-vs-pgm}

The PGM algorithm~\autocite{mckennaGraphicalmodelBasedEstimation2019} generates 
synthetic data based on the same marginal queries $a$ and noise addition as 
NAPSU-MQ.  PGM also models the original data using the $\medl$ distribution. 
Unlike NAPSU-MQ, PGM finds the parameters $\theta$ by minimising the $l_2$-distance 
$||\privs - n\mu(\theta)||_2$ between the observed noisy query values 
$\privs$ and the expected query values $n\mu(\theta) = nE_{x\sim \medl}(a(x))$
In the following, we'll replace the 
query values $s$ and $\privs$ that are summed over datapoints with 
$u = \frac{s}{n}$ and $\privu = \frac{\privs}{n}$ that represent mean query 
values over datapoints. Then the PGM objective is equivalent to 
$||\privu - \mu(\theta)||_2$.

We can view the PGM minimisation problem as a maximum likelihood estimation 
in the NAPSU-MQ probabilistic model
\begin{align}
  \xreal \sim \medl^n, \quad
  s = a(\xreal), \quad
  \privs \sim \caln(s, \sigmadp^2I), \label{eq:napsu-mq-model-appendix}
\end{align}
where we replace normal approximation that NAPSU-MQ uses with a 
law of large numbers approximation. Specifically, first replace 
$s$ with $u$ in \eqref{eq:napsu-mq-model-appendix}:
\begin{align}
  \xreal \sim \medl^n, \quad
  u = \frac{a(\xreal)}{n}, \quad
  \privu \sim \caln(u, \sigmadp^2I / n^2).\label{eq:pgm-model-appendix}
\end{align}
Because $u$ is a mean of sufficient statistics for individual datapoints,
by the law of large numbers, asymptotically $u \sim \delta_{\mu(\theta)}$.
With this approximation, the probabilistic model is 
\begin{align}
  u \sim \delta_{\mu(\theta)}, \quad
  \privu \sim \caln(u, \sigmadp^2I / n^2).
\end{align}
$u$ can be marginalised from the likelihood of this model:
\begin{align}
  p(\privu | \theta) &= \int p(\privu, u | \theta)\dx u
  \\&= \int p(\privu | u)p(u | \theta) \dx u
  \\&= \int \caln(\privu | u, \sigmadp^2I/n^2)\delta_{\mu(\theta)}(u) \dx u
  \\&= \caln(\privu | \mu(\theta), \sigmadp^2I / n^2)
\end{align}
The marginalised log-likelihood is then 
\begin{equation}
  \ln p(\privu | \theta) = -\frac{n^2}{\sigmadp^2}||\privu - \mu(\theta)||_2^2 + \text{constant},
\end{equation}
so maximising the log-likelihood is equivalent to minimising the PGM objective.

If we made a normal approximation instead of the law of large numbers 
approximation in \eqref{eq:pgm-model-appendix}, we would get
\begin{align}
  \privu \sim \caln(\mul, \Sigmal/n + \sigmadp^2I/n^2),
\end{align}
so maximising the likelihood is still possible. Unlike PGM, this maximum likelihood 
objective includes the covariance $\Sigma(\theta)$. We leave any comparisons 
between maximising this objective and PGM to future work.

\section{HYPERPARAMETERS}\label{app:hyperparams}

\paragraph{NAPSU-MQ}
The hyperparameters of NAPSU-MQ are the choice of prior, choice of inference algorithm,
and the parameters of that algorithm. For the toy data experiment, we used 
the Laplace approximation for inference, which approximates the posterior with 
a Gaussian centered at the maximum aposteriori estimate (MAP).
We find the MAP
for the Laplace approximation with the LBFGS optimisation algorithm, which we 
run until the loss improves by less than $10^{-5}$ in an iteration, up to 
a maximum of 500 iterations. Sometimes LBFGS failed to converge, which we 
detect by checking if the 
loss increased by over 1000 in one iteration, and fix by restarting optimisation 
from a different starting point. We also restarted if the maximum number of 
iterations was reached without convergence. For the vast majority of runs, no 
restarts were needed, and at most 3 were needed.

For the Adult experiment, we used NUTS~\autocite{hoffmanNoUTurnSamplerAdaptively2014}.
We ran 4 chains of 800 warmup samples and 2000 kept samples. We set the 
maximum tree depth of NUTS to 12. We normalised the posterior using the 
mean and covariance from the Laplace approximation.
For the Laplace approximation, we used the same hyperparameters as with the toy data 
set, except we set the maximum number of iterations to 6000.

For the prior, we used a Gaussian distribution with mean 0 and standard deviation
10 for all components, without dependencies between components, for both 
experiments.

\paragraph{PGM and Repeated PGM}
PGM finds the $\medl$ parameters $\theta$ that minimise the $L_2$-error 
between the expected query values and the noisy query values. The PGM 
implementation offers several algorithms for this optimisation problem, 
but we found that the default algorithm (mirror descent) and number of 
iterations works well for both experiments.

\paragraph{RAP}
RAP minimises the error on the selected queries of a continuous relaxation 
of the discrete synthetic dataset. After the optimal relaxed synthetic dataset 
is found, a discrete synthetic dataset is constructed by sampling. This gives 
two hyperparameters that control the size of the synthetic data: the size of 
the continuous dataset, and the number of samples for each datapoint in the 
continuous relaxation. We set the size of the continuous dataset to 1000 
for both experiments, as recommended by the 
paper~\autocite{aydoreDifferentiallyPrivateQuery2021}.
For the Adult data experiment, we set the number of samples per datapoint to 
46, so that the total size of the synthetic dataset is close to the 
size of the original dataset. For the toy data experiment, we set the 
number of samples per datapoint to 50. The RAP 
paper~\autocite{aydoreDifferentiallyPrivateQuery2021} finds that much smaller 
values are sufficient, but higher values should only increase accuracy.

In both cases, we weight the 
synthetic datapoints by $\frac{n}{n_{Syn}}$ before the downstream logistic 
regression to ensure that the logistic regression does not over- or underestimate 
variances because of a different sample size from the original data.

RAP also has two other hyperparameters that are relevant in our experiments: 
the number of iterations and the learning rate for the query error minimisation.
After preliminary runs, we set the learning rate at 0.1 for both experiments,
and set the number of iterations to 5000 for the toy data experiment, and 
10000 for the Adult data experiment.

\paragraph{PEP}
PEP has two hyperparameters: the number of iterations used to find 
a distribution with maximum entropy that has approximately correct query  
values, and the allowed bound on the difference of the query values.
The PEP implementation hardcodes the allowed difference to 0.
We set the number of iterations to 1000 after preliminary runs for both 
experiments.

\paragraph{PrivLCM}
PrivLCM samples the posterior of a Bayesian latent class model, where the 
number of classes in limited to make inference tractable. The model has 
hyperparameters for the prior, and the number of latent classes. We leave the 
prior hyperparameters to their defaults, and set the number of latent classes 
to 10, which the PrivLCM authors used in a 5-dimensional binary data 
experiment~\autocite{nixonLatentClassModeling2022}. The remaining hyperparameter 
of PrivLCM is the number of posterior samples that are obtained. To keep 
the runtime of PrivLCM manageable, we set the number of samples to 500 after 
ensuring that the lower number of samples did not degrade the accuracy of 
the estimated probabilities for the joint distribution compared to using the 
default of 5000 samples.

\paragraph{PrivBayes}
As we focus on synthetic data generation, we did not use the query selection 
features of PrivBayes in the toy data experiment, and instead set the 
queries according to the Bayesian network of the data generating process.
Denoting the three components of the data as $X = (X_1, X_2, X_3)$,
the queries are the full set of 1-way marginals on $X_1$, the full set of 
2-way marginals on $(X_1, X_2)$, and the full set of 3-way marginals on 
all three variables.

\section{ADULT EXPERIMENT DETAILS AND EXTRA RESULTS}\label{app:adult-details}

We include the columns Age, Workclass, Education, Marital Status, Race, Gender,
Capital gain, Capital loss, Hours per week and Income of the Adult dataset, 
and discard the rest to remove redundant columns and keep computation times 
manageable. We discretise Age and Hours per week to 5 buckets, and discretise 
Capital gain and Capital loss to binary values indicating whether their value 
is positive. The Income column is binary from the start, and indicates 
whether a person has an income $>\$50\,000$. 

In the downstream logistic regression, we use income as the dependent variable,
and Age, Race and Gender as independent variables. Age is 
transformed back to a continuous value for the logistic regression by 
picking the middle value of each discretisation bucket. We did not use all variables 
for the downstream task, as a smaller set of variables allows including the 
relevant marginals for synthetic data generation. 
The regularisation for the logistic regression is $l_2$ with a very small 
multiplier of $10^{-5}$. When the regularisation is used, variances are 
estimated with bootstrapping using 50 bootstrap samples. 

All of the Adult experiment figures, except 
Figures~\ref{fig:adult-calibration-widths-nonregularised} and 
\ref{fig:dropped-datasets} use the small regularisation term. 
Figure~\ref{fig:adult-calibration-widths-nonregularised} shows the results with the 
trick of removing large variance estimates $(\geq 10^3)$, and 
Figure~\ref{fig:dropped-datasets} shows how many estimates were removed.

For the input queries, we include the 
2-way marginals with Hours per week and each of the independent variables 
Age, Race and Gender 
and income, as well as the 2-way marginal between Race and Gender.
The rest of the queries were selected with the MST 
algorithm~\autocite{mckennaWinningNISTContest2021a}. For MST, we used 
$\epsilon = 0.5$ and $\delta = \frac{1}{n^2} \adultdelta$, 
but we do not include this in our figures, as we focus on 
the synthetic data generation, not query selection.
The selected queries are shown in Figure~\ref{fig:queries}. The selection is 
very stable: in 100 repeats of query selection, these queries were selected 
99 times.

We chose the number of generated synthetic datasets for NAPSU-MQ and 
the number of repeats for repeated PGM by comparing the results of the 
Adult experiment for different choices. The results are shown in 
Figure~\ref{fig:max-ent-m-comparison} for NAPSU-MQ and 
Figure~\ref{fig:pgm-repeats-m-comparison} for repeated PGM. We chose 
$m = 100$ for NAPSU-MQ because it produces the narrowest confidence 
intervals, and $m = 5$ repeats for repeated PGM because it had the best 
calibration overall.

\begin{figure}
  \centering
  \includegraphics[width=\linewidth]{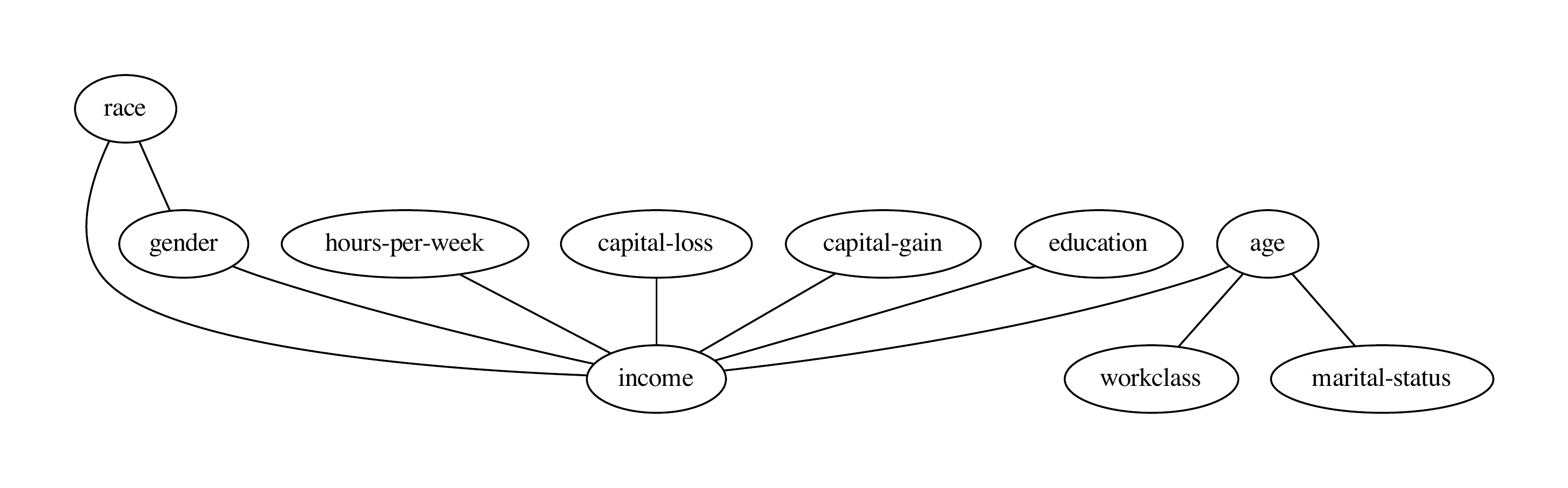}
  \caption{
    Markov network of selected queries for the Adult experiment. Each edge 
    in the graph represents a selected 2-way marginal.
  }
  \label{fig:queries}
\end{figure}

\begin{figure}[h]
  \center
  \includegraphics[width=\linewidth]{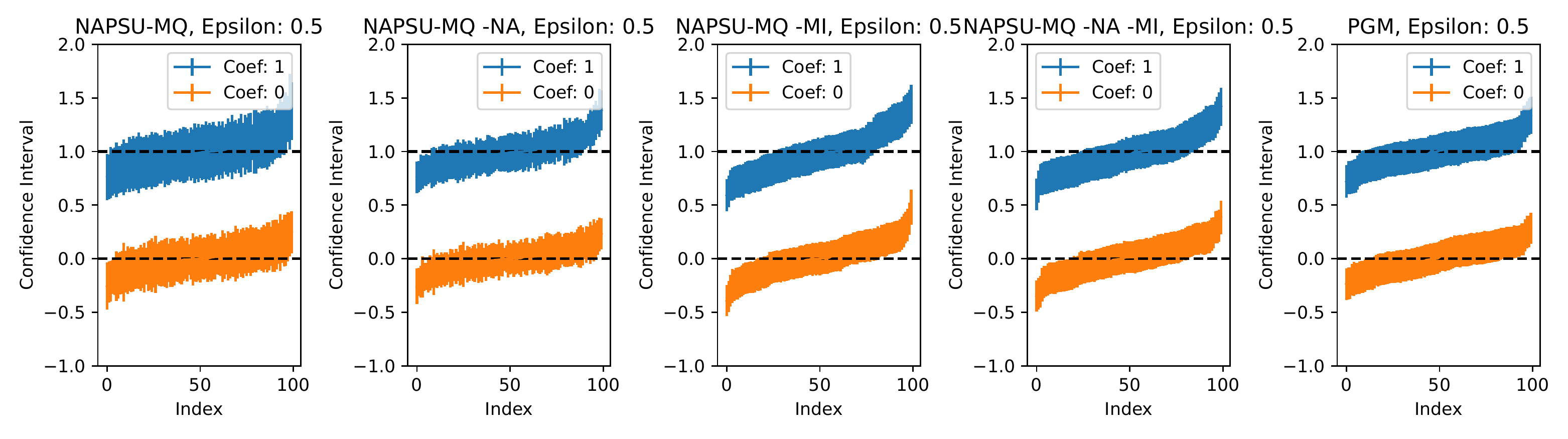}
  \caption{Ablation Study Confidence Intervals.}
  \label{fig:ablation-conf-ints}
\end{figure}

\begin{figure}[h]
  \centering
  \includegraphics[width=\linewidth]{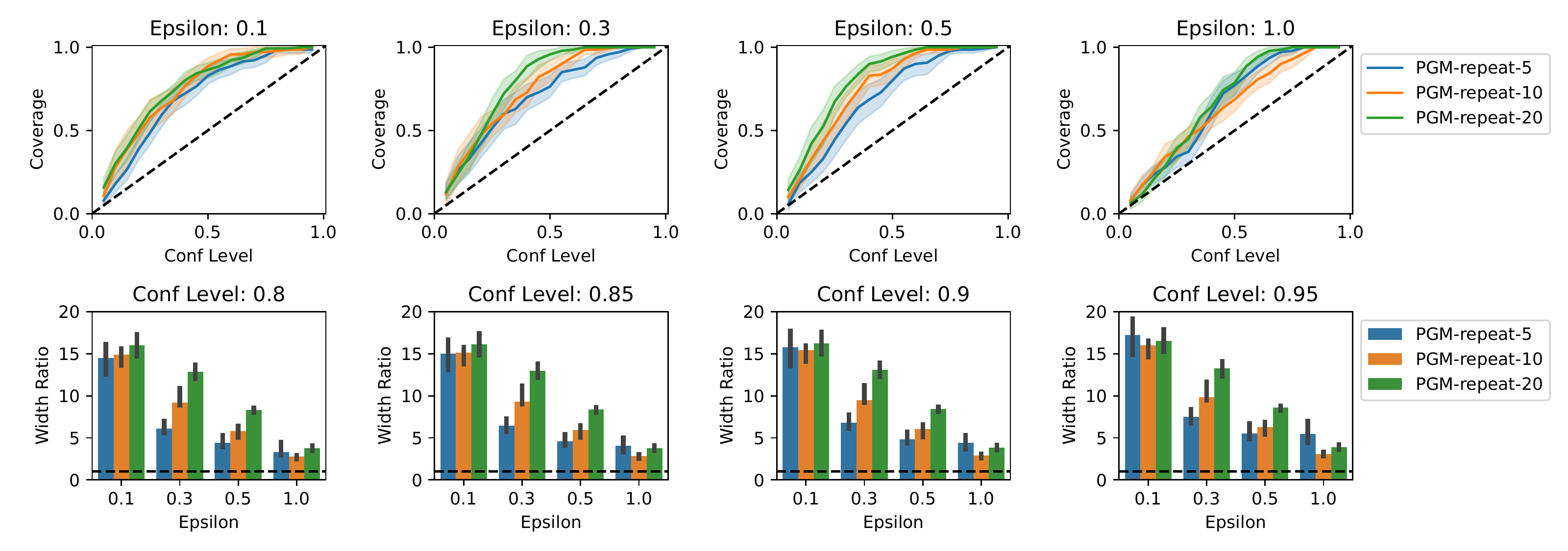}
  \caption{
    Comparison of different numbers of repetitions for repeated PGM on the Adult 
    dataset with regularisation.
    We chose $m = 5$ repeats to represent repeated PGM in the main experiment, 
    although the differences between the numbers of repeats are small.
  }
  \label{fig:pgm-repeats-m-comparison}
\end{figure}

\begin{figure}
  \centering
  \includegraphics[width=\linewidth]{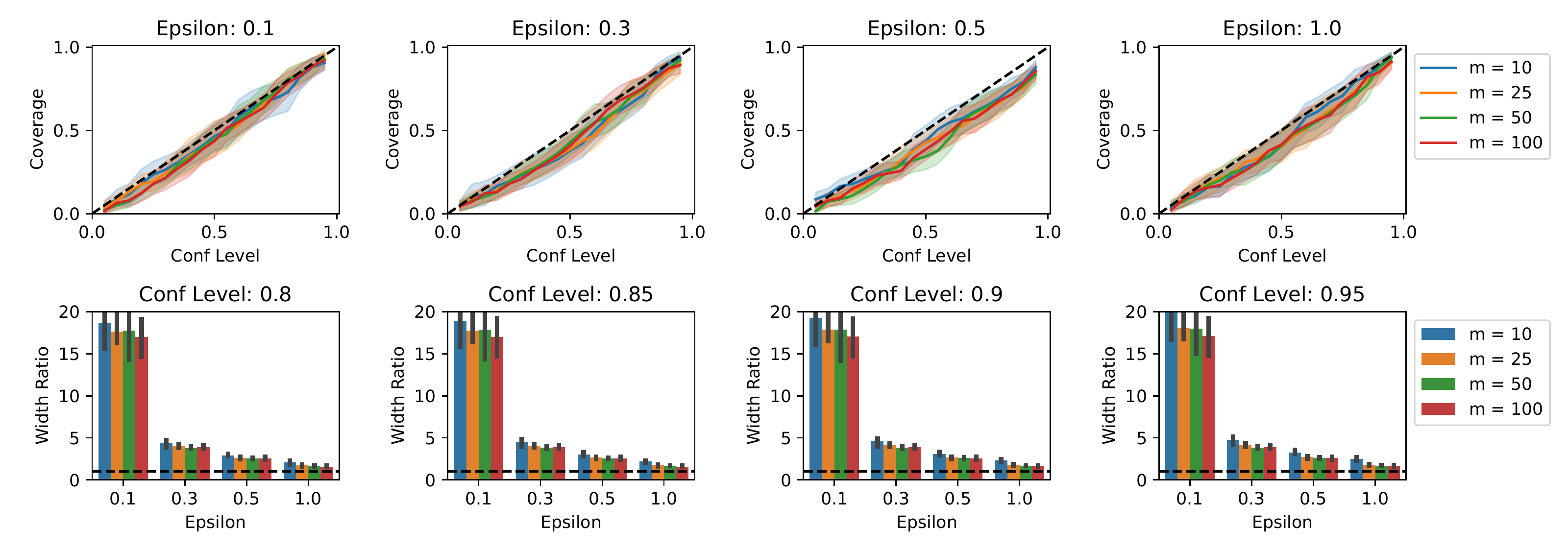}
  \caption{
    Comparison of different numbers of generated synthetic datasets for 
    NAPSU-MQ on the Adult dataset with regularisation. The differences are 
    small, but $m = 100$ synthetic datasets
    produces the narrowest intervals, so we chose it for the main experiment.
  }
  \label{fig:max-ent-m-comparison}
\end{figure}

\begin{figure}
  \centering
  \includegraphics[width=\textwidth]{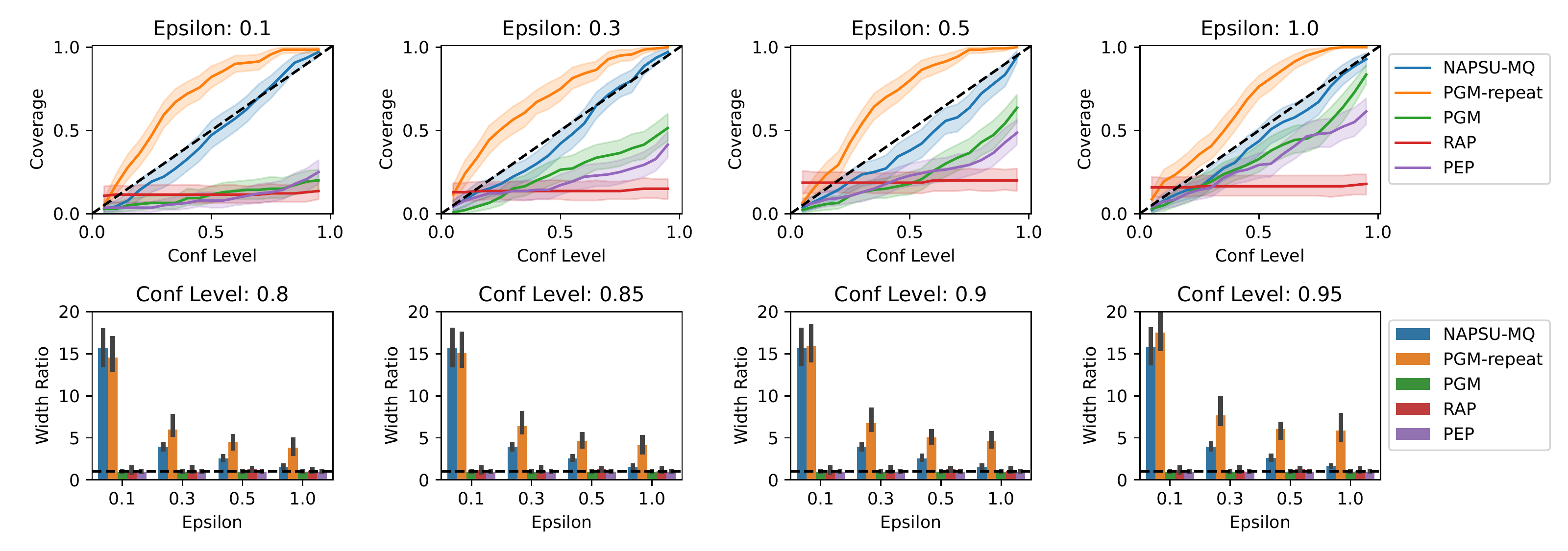}
  \caption{
    Results from the Adult data experiment with the trick of dropping large 
    variances in the logistic regression 
    instead of adding a small regularisation term. The results are almost 
    identical to Figure~\ref{fig:adult-calibration-widths}, except for RAP, 
    which suffers from the regularisation. 
  }
  \label{fig:adult-calibration-widths-nonregularised}
\end{figure}

\begin{figure}
  \centering
  \begin{subfigure}{0.48\textwidth}
    \centering
    \includegraphics[width=\textwidth]{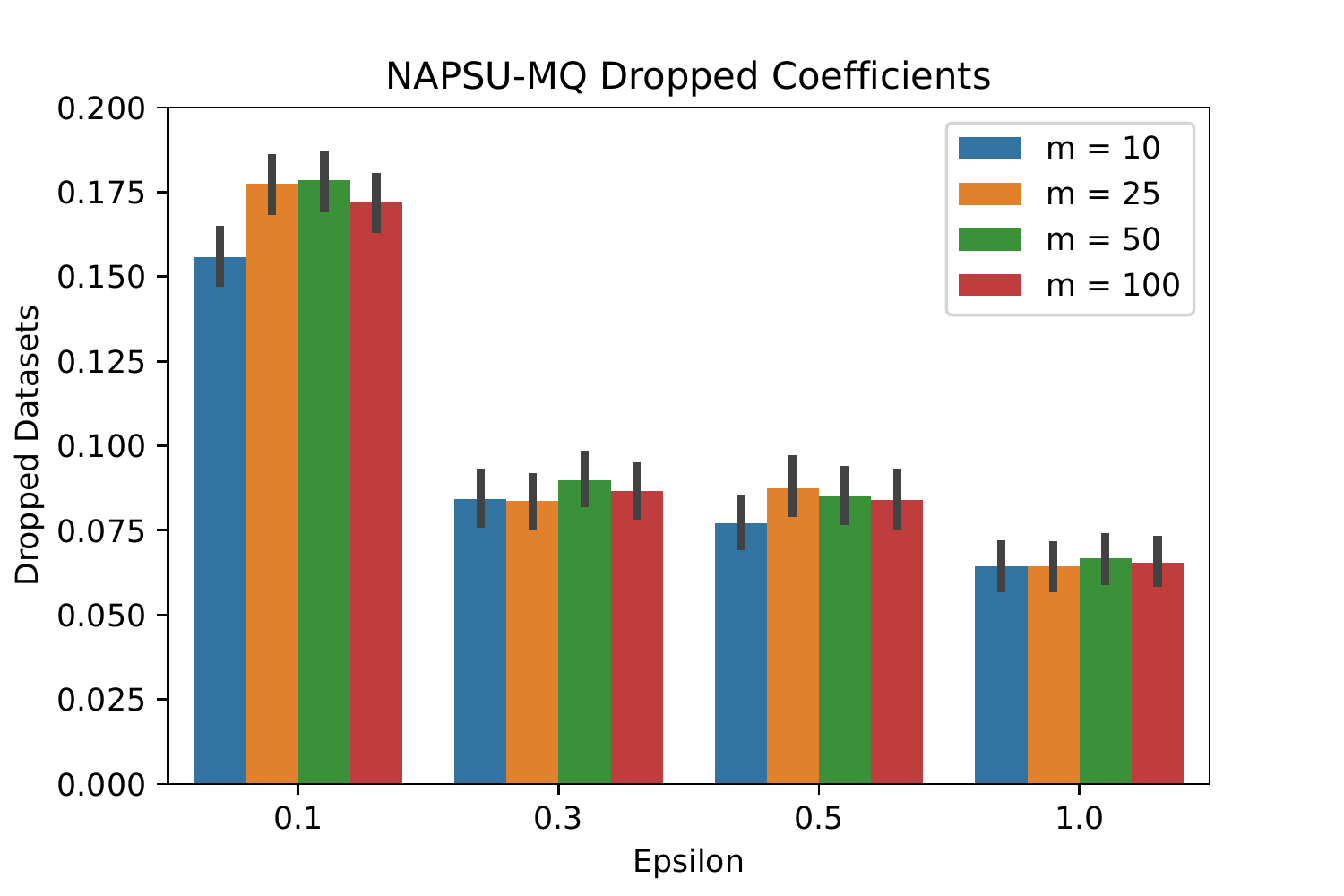}
    \caption{}
    \label{fig:napsu-mq-dropped-datasets}
  \end{subfigure}
  \begin{subfigure}{0.48\textwidth}
    \centering
    \includegraphics[width=\textwidth]{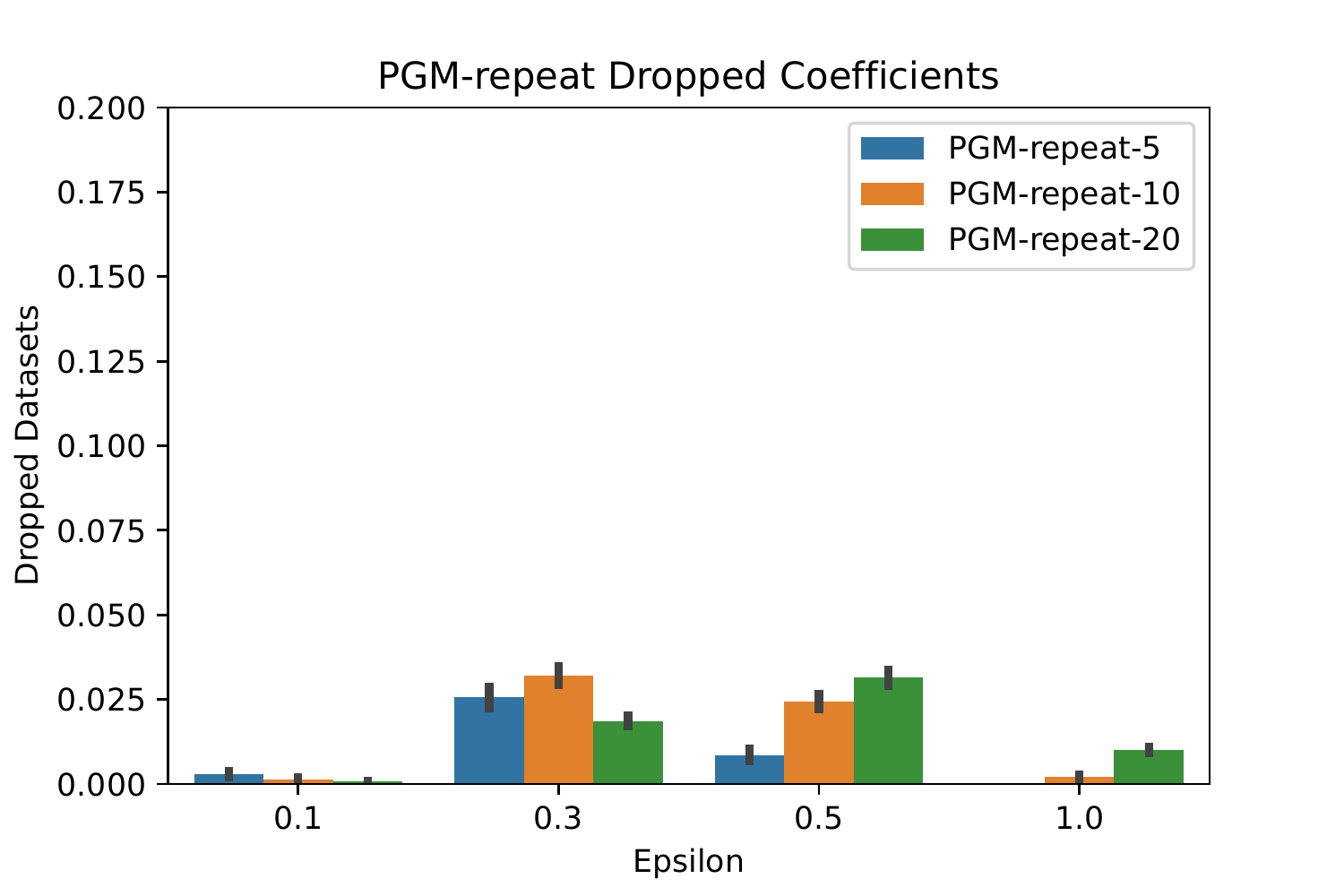}
    \caption{}
    \label{fig:pgm-dropped-datasets}
  \end{subfigure}
  \caption{
    The fraction of coefficients dropped before Rubin's rules because of very 
    high estimated variances from the downstream logistic regression in the 
    Adult data experiment for NAPSU-MQ in (a) and PGM-repeat in (b).
  }
  \label{fig:dropped-datasets}
\end{figure}

\begin{figure}
  \centering
  \includegraphics[width=\textwidth]{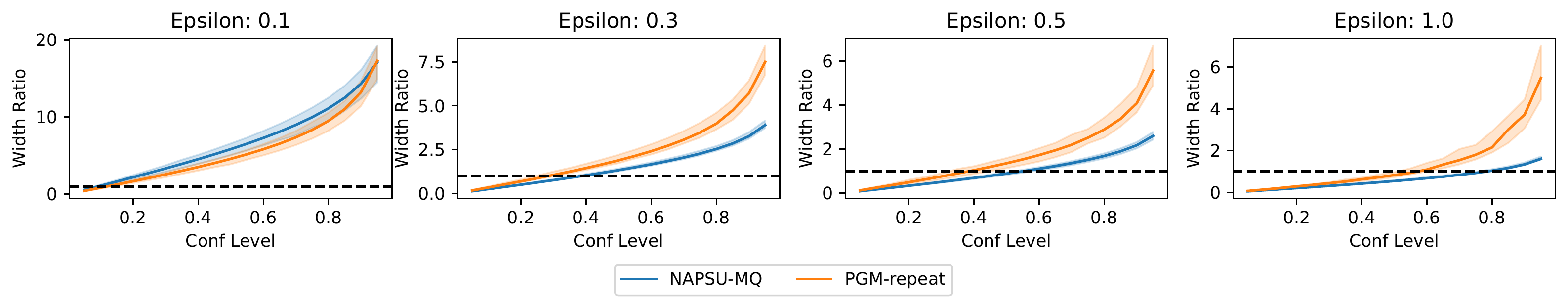}
  \caption{
    The tradeoff between the confidence level for DP confidence intervals and 
    the width of the intervals on the Adult dataset with regularisation. 
    The width ratio on the y-axis is with regards to 
    the original 95\% confidence interval, for all confidence levels, so the 
    plot shows how much must the confidence level drop to obtain the same 
    width from a DP confidence interval as a non-DP one. The horizontal line
    at $y = 1$ shows this point. For $\epsilon = 1$, the confidence 
    level for NAPSU-MQ must be dropped to about 75\%, and for PGM-repeat, 
    it must be dropped to about 50\%.
  }
  \label{fig:conf_level_width_95_tradeoff}
\end{figure}

\begin{figure}
  \centering 
  \includegraphics[width=\textwidth]{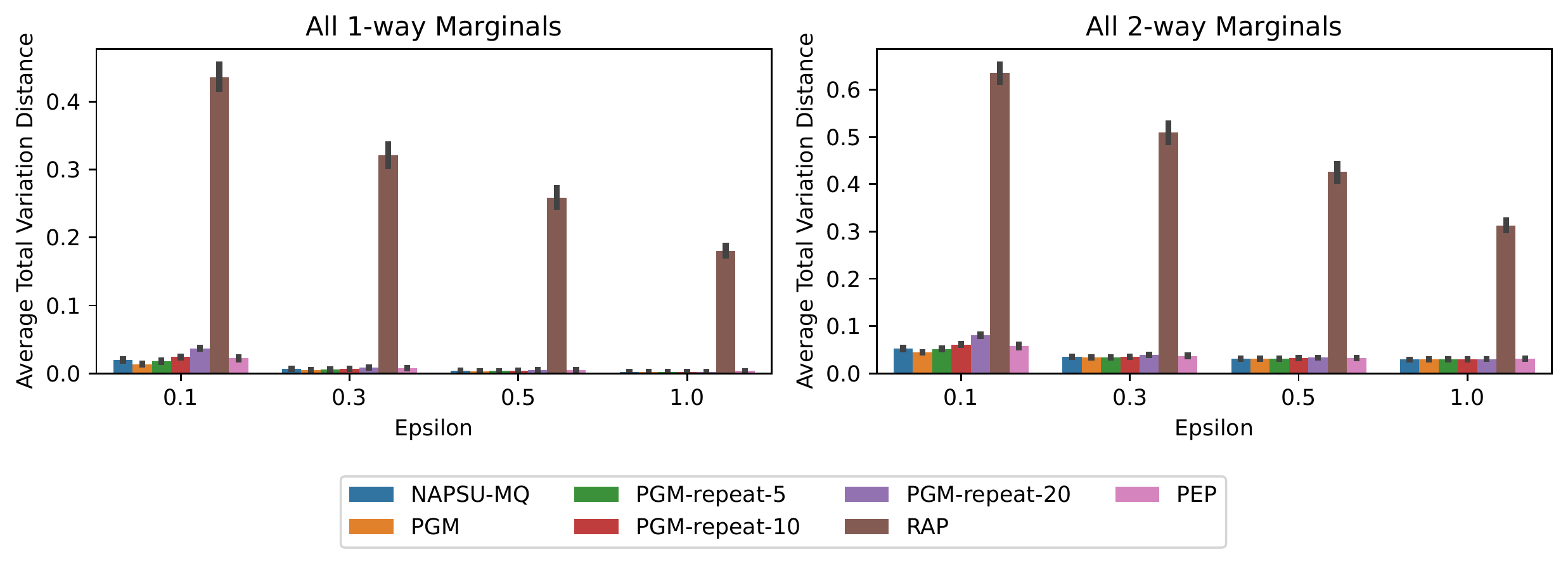}
  \caption{
    Comparison of marginal query accuracy for Adult data. 
    NAPSU-MQ is almost as accurate as PGM for all values of $\epsilon$, and 
    is on par with PGM-repeat. The panels 
    show the average total variation distance of all 1-way marginal 
    distributions (left) or all 2-way marginal distributions (right) between 
    the original discretised data and synthetic data, averaged over 20 repeats.
    For NAPSU-MQ and PGM-repeat-$m$, the synthetic marginal distributions were 
    estimated by averaging over $m$ synthetic datasets, with $m = 100$ for 
    NAPSU-MQ.
  }
  \label{fig:adult_marginal_accuracy}
\end{figure}

\begin{table}
\centering
\caption{
Runtimes of each inference run for the Adult experiment. Does not include the time 
taken to generate synthetic data, or run any downstream analysis. The LA rows record the 
runtime for obtaining the Laplace approximation for NAPSU-MQ that is used in the NUTS inference, so the 
total runtime for a NAPSU-MQ run with NUTS is the sum of the LA and NUTS rows.
Experiments were run on 4 CPU cores of a cluster.
}
\label{tab:runtime}
\begin{tabular}{llll}
\toprule
    &     &             Mean & Standard Deviation \\
Algorithm & Epsilon &                  &                    \\
\midrule
\multirow{4}{*}{LA} & 0.1 &       2 min 53 s &             18.5 s \\
    & 0.3 &       3 min 53 s &             29.4 s \\
    & 0.5 &       3 min 38 s &             35.0 s \\
    & 1.0 &       3 min 25 s &             25.5 s \\
\cline{1-4}
\multirow{4}{*}{NUTS} & 0.1 &   9 h 59 min 6 s &             6506 s \\
    & 0.3 &  7 h 33 min 28 s &             2701 s \\
    & 0.5 &  4 h 57 min 40 s &             3185 s \\
    & 1.0 &  3 h 51 min 34 s &             1274 s \\
\cline{1-4}
\multirow{4}{*}{PEP} & 0.1 &       6 min 50 s &             25.4 s \\
    & 0.3 &       7 min 18 s &             31.2 s \\
    & 0.5 &        7 min 0 s &             33.1 s \\
    & 1.0 &        7 min 7 s &             33.7 s \\
\cline{1-4}
\multirow{4}{*}{PGM} & 0.1 &             15 s &              0.5 s \\
    & 0.3 &             17 s &              1.5 s \\
    & 0.5 &             15 s &              0.4 s \\
    & 1.0 &             15 s &              0.6 s \\
\cline{1-4}
\multirow{4}{*}{PGM-repeat-10} & 0.1 &       2 min 35 s &              3.3 s \\
    & 0.3 &       2 min 53 s &             13.0 s \\
    & 0.5 &       2 min 37 s &              5.0 s \\
    & 1.0 &       2 min 36 s &              4.4 s \\
\cline{1-4}
\multirow{4}{*}{PGM-repeat-20} & 0.1 &       5 min 15 s &             10.9 s \\
    & 0.3 &       5 min 58 s &             28.4 s \\
    & 0.5 &       5 min 10 s &             10.2 s \\
    & 1.0 &       5 min 13 s &             12.6 s \\
\cline{1-4}
\multirow{4}{*}{PGM-repeat-5} & 0.1 &       1 min 17 s &              2.7 s \\
    & 0.3 &       1 min 28 s &              6.7 s \\
    & 0.5 &       1 min 18 s &              2.6 s \\
    & 1.0 &       1 min 18 s &              1.9 s \\
\cline{1-4}
\multirow{4}{*}{RAP} & 0.1 &             32 s &              2.4 s \\
    & 0.3 &             34 s &              2.2 s \\
    & 0.5 &             32 s &              2.1 s \\
    & 1.0 &             31 s &              2.1 s \\
\bottomrule
\end{tabular}
\end{table}

\section{US CENSUS DATA EXPERIMENT}\label{app:us-census-experiment}

We conducted an additional experiment on US Census data from the 
UCI repository~\autocite{meekUSCensusData2001}. We limited the data to individuals 
who have served in the US Military, and picked 9 columns\footnote{
  The columns are dYrsserv, iSex, iVietnam, iKorean, iMilitary, dPoverty, 
  iMobillim, iEnglish and iMarital.
}, most relating to military service. Even this subset of the data 
is large, with $n = 320\,754$. All columns are discrete, and have $10\,800$
possible values, much fewer than the Adult experiment. 

As the downstream task,
we use logistic regression with dPoverty as the dependent variable and 
iSex, iKorean, iVietnam and iMilitary as the independent variables. dPoverty 
has three categories, so we combine the two categories denoting people below 
the powerty line to make the dependent variable binary for the logistic 
regression, but not synthetic data generation.

As our queries we use 4 three-way marginals covering the independent and 
dependent variables, and 3 two-way marginals that include the other variables 
that are synthesised, but not included in the regression. As the published 
implementation of RAP~\autocite{aydoreDifferentiallyPrivateQuery2021} does not 
support a mix of two- and three-way marginals, we replace the two-way marginals 
with three-way marginals for RAP. As in the Adult experiment, we set 
$\delta = n^{-2} \uscensusdelta$, and vary $\epsilon$.

As in the adult experiment, we use $n_{Syn} = n$ for all algorithms except RAP.
For PGM-repeat and NAPSU-MQ, we choose $m$ with a preliminary experiment.
For NAPSU-MQ, we set $m = 100$, although the differences between the choices 
are not large. For PGM-repeat, we set $m = 10$. We set the other hyperparameters 
for all algorithms after testing runs to the same values used in the Adult 
experiment, except we increased the number of optimisation iterations 
for PGM to 10\,000 from the default of 1000, the number of iterations for PEP to 
10\,000 from 1000, and increased the number of 
kept samples in NUTS to 4000 for NAPSU-MQ.
We did not use the trick of dropping estimates with very high variances, 
or using very small regularisation in the logistic regression with the US 
Census data.

The results are shown in Figure~\ref{fig:us-census-coverage-lengths}. 
While PGM is calibrated with $\epsilon \geq 0.3$, it is severely overconfident with 
$\epsilon = 0.1$. This is likely caused by the large size of the dataset: at 
larger values of $\epsilon$, there is little noise compared to the large 
sample size, while at $\epsilon = 0.1$, the noise has a clear effect.

NAPSU-MQ and PGM-repeat are able to produce calibrated results at 
$\epsilon = 0.1$. Of these, NAPSU-MQ produces clearly narrower confidence 
intervals for all values of $\epsilon$.

Figure~\ref{fig:us-census-marginal-accuracy} shows the accuracies of the 
produced synthetic datasets on all 1-way and 2-way marginal queries
for the algorithms. As with the Adult dataset, shown in 
Figure~\ref{fig:adult_marginal_accuracy}, NAPSU-MQ is almost as accurate as 
PGM, and is equally accurate as PGM-repeat. For $\epsilon \geq 0.5$, all of the 
aforementioned algorithms are almost as accurate. RAP and PEP are nowhere close 
to these algorithms in accuracy, having errors that are several times larger 
than the other algorithms.

For some reason, PEP fails completely with this dataset. We are not sure what 
causes this, as the algorithm should work in this setting as well as it did 
with the Adult dataset, and the size of the dataset should not be an issue.

The runtimes for each algorithm are shown in Table~\ref{tab:runtime-us-census}.
The difference between PGM-repeat and NAPSU-MQ is much smaller than in the 
Adult data experiment, but is still large. 

\begin{figure}
  \centering
  \includegraphics[width=\textwidth]{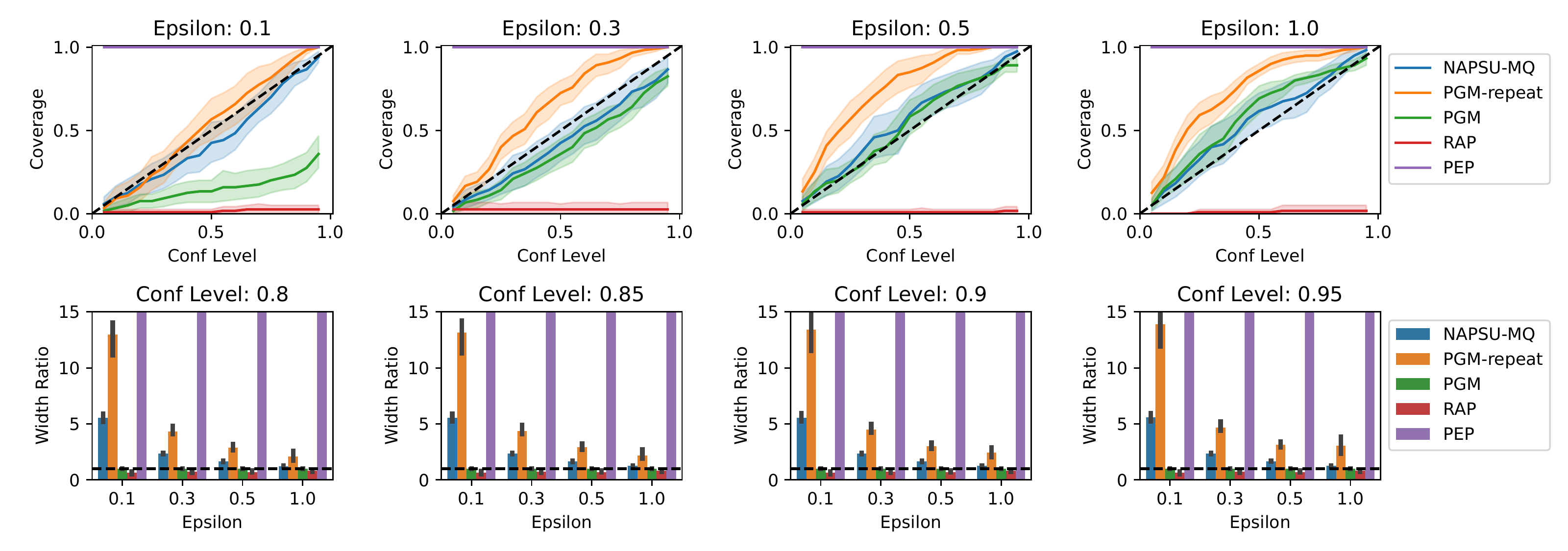}
  \caption{
    Results for the US Census experiment, showing that only NAPSU-MQ and 
    PGM-repeat are calibrated for all values of $\epsilon$, and NAPSU-MQ 
    produces significantly narrower confidence intervals than 
    PGM-repeat. Like Figure~\ref{fig:adult-calibration-widths}, the top 
    row shows the mean coverage over all coefficients and 20 runs for 
    different confidence levels. The bottom row shows median confidence interval 
    widths divided by real data confidence interval widths.
    $\delta \uscensusdelta$ in all panels.}
  \label{fig:us-census-coverage-lengths}
\end{figure}

\begin{figure}
  \centering
  \includegraphics[width=\textwidth]{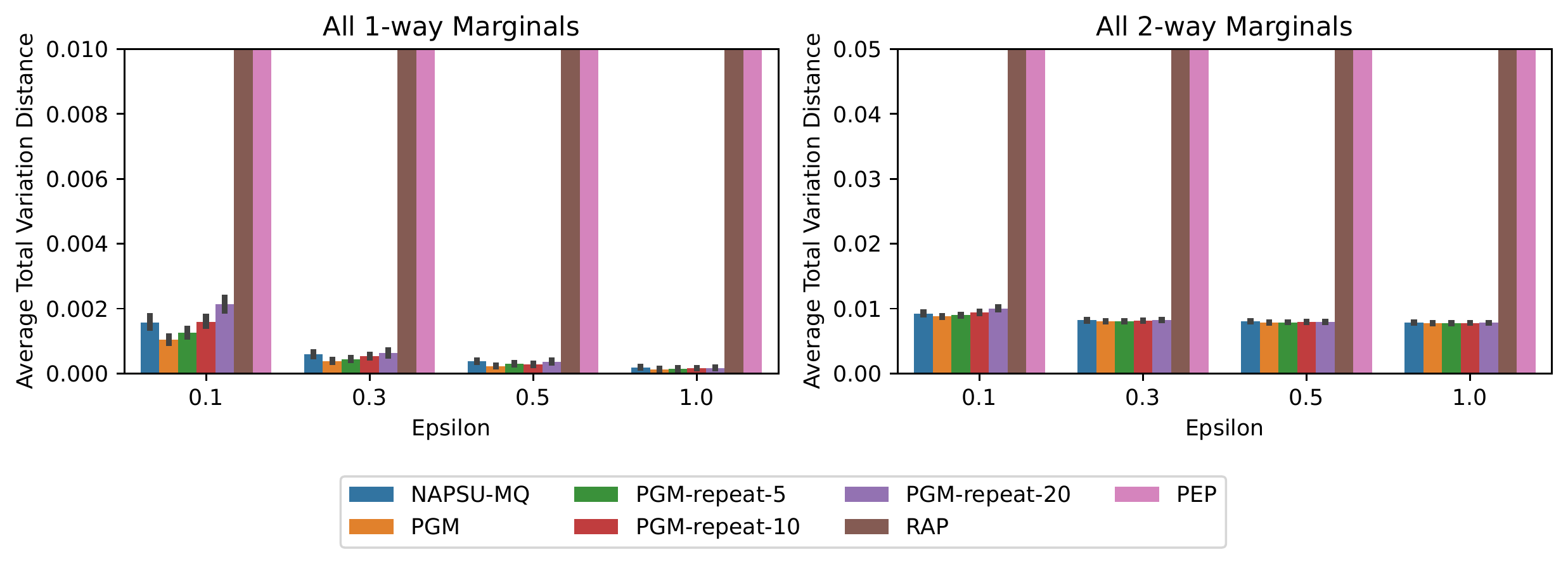}
  \caption{
    Comparison of marginal query accuracy for US Census data. 
    NAPSU-MQ is almost as accurate as PGM for all values of $\epsilon$, and 
    is on par with PGM-repeat, as with the Adult data in 
    Figure~\ref{fig:adult_marginal_accuracy}. The panels 
    show the average total variation distance of all 1-way marginal 
    distributions (left) or all 2-way marginal distributions (right) between 
    the original discretised data and synthetic data, averaged over 20 repeats.
    For NAPSU-MQ and PGM-repeat-$m$, the synthetic marginal distributions were 
    estimated by averaging over $m$ synthetic datasets, with $m = 100$ for 
    NAPSU-MQ. RAP and PEP have average total variation distances over 0.1 for 
    both 1-way and 2-way marginals for all values of $\epsilon$.
  }
  \label{fig:us-census-marginal-accuracy}
\end{figure}

\begin{table}
\centering
\caption{
Runtimes of each inference run for the US Census experiment. Does not include the time 
taken to generate synthetic data, or run any downstream analysis. The LA rows record the 
runtime for obtaining the Laplace approximation for NAPSU-MQ that is used in the NUTS inference, so the 
total runtime for a NAPSU-MQ run with NUTS is the sum of the LA and NUTS rows.
Experiments were run on 4 CPU cores of a cluster.
}
\label{tab:runtime-us-census}
\begin{tabular}{llll}
\toprule
    &     &             Mean & Standard Deviation \\
Algorithm & Epsilon &                  &                    \\
\midrule
\multirow{4}{*}{LA} & 0.1 &        2 min 2 s &             76.8 s \\
    & 0.3 &       1 min 35 s &             23.8 s \\
    & 0.5 &       1 min 51 s &             44.0 s \\
    & 1.0 &       1 min 51 s &             44.8 s \\
\cline{1-4}
\multirow{4}{*}{NUTS} & 0.1 &  3 h 32 min 25 s &             2836 s \\
    & 0.3 &  1 h 57 min 45 s &              989 s \\
    & 0.5 &  1 h 31 min 15 s &              951 s \\
    & 1.0 &    1 h 8 min 6 s &              477 s \\
\cline{1-4}
\multirow{4}{*}{PEP} & 0.1 &             17 s &              0.6 s \\
    & 0.3 &             18 s &              1.0 s \\
    & 0.5 &             17 s &              0.4 s \\
    & 1.0 &             17 s &              0.5 s \\
\cline{1-4}
\multirow{4}{*}{PGM} & 0.1 &       1 min 57 s &              2.8 s \\
    & 0.3 &       1 min 58 s &              3.2 s \\
    & 0.5 &       1 min 57 s &              4.2 s \\
    & 1.0 &       1 min 57 s &              3.2 s \\
\cline{1-4}
\multirow{4}{*}{PGM-repeat-10} & 0.1 &      19 min 22 s &             28.0 s \\
    & 0.3 &      19 min 18 s &             21.9 s \\
    & 0.5 &      19 min 36 s &             33.9 s \\
    & 1.0 &      19 min 25 s &             22.0 s \\
\cline{1-4}
\multirow{4}{*}{PGM-repeat-20} & 0.1 &      38 min 38 s &             50.2 s \\
    & 0.3 &      38 min 59 s &             37.5 s \\
    & 0.5 &      38 min 57 s &             40.5 s \\
    & 1.0 &      38 min 39 s &             74.7 s \\
\cline{1-4}
\multirow{4}{*}{PGM-repeat-5} & 0.1 &       9 min 45 s &             17.6 s \\
    & 0.3 &       9 min 38 s &              8.9 s \\
    & 0.5 &       9 min 45 s &             12.8 s \\
    & 1.0 &       9 min 43 s &              9.8 s \\
\cline{1-4}
\multirow{4}{*}{RAP} & 0.1 &             28 s &              2.4 s \\
    & 0.3 &             28 s &              2.0 s \\
    & 0.5 &             27 s &              1.2 s \\
    & 1.0 &             27 s &              3.5 s \\
\bottomrule
\end{tabular}
\end{table}

\end{document}